\documentclass[conference]{IEEEtran}
\IEEEoverridecommandlockouts

\usepackage[utf8]{inputenc}
\usepackage[T1]{fontenc}

\usepackage{booktabs} %
\usepackage{etoolbox} %
\usepackage{subfig}
\usepackage{bbm}
\usepackage{grffile}
\usepackage{url}
\usepackage{amsmath}
\usepackage{amsfonts}
\usepackage{amsthm}
\usepackage{xspace}
\usepackage{enumitem}
\usepackage{mathtools}
\usepackage[ruled,resetcount,noend,linesnumbered]{algorithm2e}
\usepackage{comment}
\usepackage{tikz}
\usepackage{enumitem}
\usepackage{tablefootnote}
\usepackage[flushleft]{threeparttable} %
\usepackage{multirow}
\usepackage{multicol}
\usepackage{arydshln}
\usepackage{blindtext}
\usepackage{balance}

\definecolor{pastelblue}{RGB}{76,113,175}
\definecolor{pastelgreen}{RGB}{84,167,104}
\definecolor{pastelred}{RGB}{196,78,82}
\definecolor{pastelgrey}{RGB}{230,230,230}
\definecolor{pastelbeige}{RGB}{243,236,221}
\definecolor{pastelpurple}{RGB}{154,139,192}
\definecolor{mysalmon}{RGB}{250, 128, 114}
\definecolor{myblue}{RGB}{60,105,210}
\definecolor{mygreen}{RGB}{60,179,113}

\newcommand\Mark[1]{\textsuperscript{#1}}
\newcommand{\mypm}{\mathbin{\mathpalette\@mypm\relax}}

\newcommand{\sS}{\mathcal{S}}
\newcommand{\bW}{{\bf W}}

\newcommand{\eat}[1]{}
\newtheorem{definition}{Definition}
\newtheorem{lemma}{Lemma}
\newtheorem{theorem}{Theorem}

\AfterEndEnvironment{tabular}{\tabularendstuff}
\newcommand{\tabularendstuff}{}

\newcommand{\name}{R-GPM\xspace} %

\def\BibTeX{{\rm B\kern-.05em{\sc i\kern-.025em b}\kern-.08em
    T\kern-.1667em\lower.7ex\hbox{E}\kern-.125emX}}

\begin{document}
\title{Graph Pattern Mining and Learning through \\User-defined Relations (Extended Version)
  \\[.75ex] 
  {\normalfont\large 
Carlos H. C. Teixeira\Mark{*},
Leonardo Cotta\Mark{\#},
Bruno Ribeiro\Mark{\#},
Wagner Meira Jr.\Mark{*}
  }
  \\[-1.5ex]
  }
\author{
\IEEEauthorblockA{
\Mark{*}\textit{Department of Computer Science} \\
\textit{Universidade Federal de Minas Gerais}\\
Belo Horizonte, Brazil \\
\{carlos,meira\}@dcc.ufmg.br} \and
\IEEEauthorblockA{
\Mark{\#}\textit{Department of Computer Science} \\
\textit{Purdue University}\\
West Lafayette, IN \\
\{cotta,ribeirob\}@purdue.edu}
}
\maketitle

\begin{abstract}
In this work, 
we propose \name, a parallel computing framework 
for graph pattern mining (GPM) through a user-defined subgraph relation.
More specifically, we enable the computation of pattern statistics 
through their subgraph classes, generalizing traditional GPM methods.
\name provides efficient estimators for these statistics 
by employing an MCMC sampling algorithm combined with several optimizations. 
We provide both theoretical guarantees and empirical evaluations of our estimators in application scenarios such as stochastic optimization of 
higher-order graph neural network models and pattern (motif) 
counting. %
We also propose and evaluate optimizations 
that enable improvements 
of our estimators accuracy, while reducing their
computational costs in up to 3-orders-of-magnitude. 
Finally, we show that \name is scalable, providing near-linear speedups on 44 cores
in all of our tests.

\end{abstract}

\section{Introduction}
\label{sec:intro}

\newcommand{\CIS}{CIS\xspace}
\newcommand{\CISes}{CISes\xspace}
\newcommand{\IS}{IS\xspace}
\newcommand{\ISes}{ISes\xspace}

Graph pattern mining (GPM) consists of 
finding \emph{relevant} patterns in labeled graphs (networks\footnote{Throughout this work we will use the terms {\em graph} and {\em network} interchangeably.}).
A pattern is a template of subgraphs, say, two females and two males in a social network connected as a fully connected 4-node subgraph.
Further, the relevance of a pattern is given by the properties of 
its occurrences in the graph.
For example, in frequent subgraph mining,
 a subgraph template is considered relevant 
if it occurs frequently in the network, i.e., 
the relevance criterion is based on the popularity of the graph patterns~\cite{kuramochi2005finding,preti2018beyond}.%

In the last decades, a variety of GPM methods have emerged with the purpose to speedup
existing
algorithms~\cite{jiang2013survey,buzmakov2017efficient}.
These techniques are used in several high-impact applications, such as label
and link prediction~\cite{Meng2018} and the analysis of biological
networks~\cite{alon-network-motifs}, semantic (RDF)
graphs~\cite{Choi:2009}, citation and social
networks~\cite{percolation,Benson2016}. More recently, there has
been growing interest in specific properties of these patterns~\cite{corescope,
koutra2015summarizing}.

Despite the huge interest in GPM %
applications, 
existing methods are usually restricted to specific tasks.
Moreover, designing an efficient algorithm 
for a given GPM task is usually hard,
since even reasonably-sized real-world networks ($>10$k nodes) 
tend to have a massive number of $k$-node subgraphs ($k \geq 4$),
which may let the mining process impracticable. 

This work generalizes  GPM tasks through user-defined local \emph{subgraph relations}, 
and introduces an efficient sampling algorithm to estimate user-defined subgraph statistics over 
the subgraph classes that arise from these relationship definitions.
Let  $\sS^{(k)}$ be the set of all $k$-node induced subgraphs of $G$.
A subgraph relation R splits  $\sS^{(k)}$ into partitions or \emph{subgraph classes}, 
where a pair of subgraphs $S, S' \in \sS^{(k)}$ belongs to the same {\em subgraph class} iff they have relationship R.
Our task is to compute the relevance of a pattern as a function of its subgraph classes in $G$.

From the practitioner's point of view, relations provide strong advantages over traditional methods:
(a) a novel pattern analysis task is defined by simply setting a new subgraph relation;
(b) relations are flexible and several subgraph relations may be developed by a user.
For instance, a relation may consider the spatial location of the subgraphs, the attributes of their nodes and edges, or even complex networks metrics (e.g., centrality, clustering coefficient, etc);
(c) subgraph relations organize $\sS^{(k)}$ into partitions, which may be used to understand and interpret the relevant patterns reported by the algorithm. 
  Finally, (d) they generalize the existing GPM solutions since each subgraph can be  considered as a class itself, i.e., $S\mathrm{R}S'$ is true iff $S = S'$.

To the best of our knowledge, there is no efficient algorithms able to compute statistics of 
patterns based on user-defined subgraph relations.
 A na\"ive method 
 would have to (1) enumerate the subgraph set $\sS^{(k)}$  and, (2) doubly-iterate over $\sS^{(k)}$ testing $S$R$S'$ for all pairs of $S,S' \in \sS^{(k)}$, 
which is computationally intractable even in moderate-sized graphs.

\noindent
{\bf Contributions:}
This paper introduces a generalized GPM task and an efficient and parallel 
sampling framework, \name, to estimate relevance criteria for a large family of user-defined relations and summarization functions.
Our method computes statistics of the subgraph classes 
by integrating a computationally bounded exact algorithm
with an unbiased estimator based on random walks, 
through a novel use of the renewal-reward theorem of Markov chains.
More specifically, \name is able to take advantage 
of an incomplete subgraph class' 
computation to improve the estimator accuracy in two complementary ways: 
(a) by only estimating the residual that has not been computed exactly, and 
(b) by parallelizing and reducing the variance associated with random walk sampling using 
the subgraphs of the exact computation as a stopping set for the random walk.
%
%

\begin{comment}
The main innovation of our approach is the integration of a computationally bounded exact computation with an unbiased estimator, 
which is activated once the exact computation budget is exhausted. 

Through a novel use of the renewal-reward theorem of Markov chains, 
our random walk subgraph-function estimator is able to take advantage 
of this incomplete exact computation to improve the estimator accuracy in two complementary ways: 
(a) by only estimating the residual that has not been computed exactly, and 
(b) by parallelizing and reducing the variance associated with random walk sampling using 
the subgraphs of the exact computation as a stopping set for the random walk.
As we will see latter, this strategy reduces 
 the cost of running our estimator by up to 3 orders-of-magnitude.
\end{comment}

%

%
We also introduce a host of innovations in random walk sampling for subgraph relevance score estimation, 
such as non-backtracking random walks, data structure optimizations and parallel unbiased estimators using random walks.
In particular, \name adopts a \emph{producer-consumer} parallelization model which provides 
a near-linear speedup on 44 cores.
Moreover, we show that subgraph relations can be useful in
 tasks ranging from a Robbins-Monro~\cite{Robbins51} stochastic optimization method to train deep neural networks for subgraph evolution prediction,
 to a generalization of $k$-clique percolation using any $k$-connected subgraph~\cite{PalEtAl05}.

\noindent
{\bf Reproducibility:}
Our open-sourced code and the data we
used are at \url{http://github.com/dccspeed/rgpm}.

\section{Preliminaries}\label{s:pre}

GPM problems take an attributed undirected graph $G = (V,E,\Phi)$,
where $V$ (also denoted as $V(G)$) is defined as the set of vertices, $E$ is the set of edges ($E(G)$) and $\Phi(v,G)$ is the label of node $v \in V$ in $G$. 
Let $S$ be a \emph{connected and induced subgraph} (\CIS) in $G$. 
$S$ is induced if it has a set of nodes $V(S) \subset V(G)$ and a set of edges
$E(S)$ containing all edges in $E(G)$ that have both endpoints in $V(S)$.
In addition, $S$ is connected when there is a path between any pair of nodes in $V(S)$ formed by the edges in $E(S)$. 
The subgraphs used in this work are all \CISes. 

Roughly speaking, a pattern is a graph template.
We define a pattern (or canonical representation)
of a subgraph $S$, $\rho(S)$, as its canonical labelling code~\cite{junttila2007engineering}.
Therefore, if two subgraphs $S$ and $S'$ are \emph{isomorphic}, then, $\rho(S) = \rho(S')$.
Note that our approach is not tied to this labelling
and other representation form can be used (e.g.~\cite{kuramochi2001frequent,huan2003efficient}).

\noindent
{\bf Problem statement:}
Given a user-defined relation R, computing the pattern relevance
score demands the instantiation of three functions: (i) $g$, the subgraph
function; (ii) $\alpha$, the class weight function; and (iii) $F$, the pattern statistic function. 
$g$ is an arbitrary user-defined function ($|g(\cdot)| < \infty$) and, 
intuitively, it quantifies the contribution of a subgraph to its pattern's score. 
As we will see, $g$ is the basis of both $\alpha$ and $F$ functions. 

The first part of our problem is to compute the {\em class weight function}, $\alpha(\cdot)$,  for 
the class of given subgraph $S$,
 iterating over all \CISes in $\sS^{(k)}$ as follow:
\begin{equation}\label{eq:alpha}
\alpha(\mathrm{R}, S) = \sum_{S' \in \sS^{(k)}}  g(S') \cdot {\bf 1}_{\{S'\mathrm{R}S\}},
\end{equation}
where   
${\bf 1}_{\{S'\mathrm{R}S\}} = 1$ iff subgraphs $S',S \in \sS^{(k)}$ have relationship R, otherwise ${\bf 1}_{\{S'\mathrm{R}S\}} = 0$.

Finally, the last part of our problem is to compute the relevance score of a pattern $P$ in a graph $G$.
This is given by the {\em pattern statistic function}, $F$, which is given by:
\begin{equation}\label{eq:Fgeneral}
F(G, \mathrm{R}, P) = \frac{1}{\lambda} 
\sum_{{S \in \sS^{(k)}}}
\frac{\alpha(\mathrm{R},S)}{|C_{\mathrm{R},S}|}\cdot {\bf 1}_{\{\rho(S) = P\}}, 
\end{equation}
where $\rho(S)$ is a function that gives the pattern of subgraph $S$,
 $C_{\mathrm{R}, S}$ is the set of subgraphs related to $S$ or, formally, 
$C_{\mathrm{R},S} = \{ S' \in \sS^{(k)} | 1_{\{S'\mathrm{R}S\}} = 1\}$ and
 $\lambda$ 
is the normalization factor equal to $\sum_{{S \in \sS^{(k)}}}\frac{\alpha(\mathrm{R},S)}{|C_{\mathrm{R},S}|}$. 
Trivially, if $g(S)=\frac{1}{|C_{\mathrm{R},S}|}$ 
$F$ returns the proportion of classes 
that a pattern $P$ has in $G$, while 
the standard motifs counting problem arises
for $g(\cdot)=1$ and $1_{\{S'\mathrm{R}S\}}=1$ iff $S' = S$.

Note that,  computing both $\alpha$ and $F$ 
are computationally expensive since they require a sum over all \CISes in $\sS^{(k)}$. 
Our framework overcomes this issue by estimating $\alpha$ and 
$F$ through MCMC sampling (as shown in Section \ref{sec:frame}).

\noindent 
\textbf{Equivalence-isomorphic relation R.} \label{sec:relations}
Despite the variety of possible relations R, 
this work focuses on a special 
group, which we denote \emph{equivalence-isomorphic} relations.
\begin{definition}
[Equivalence-isomorphic relation] 
We define an equivalence-isomorphic relation $\mathrm{R}$ as satisfying the following properties:
(1) reflexivity: a relation $\mathrm{R}$ is reflexive if for all $S \in \sS^{(k)}$, $S\mathrm{R}S$;
(2) symmetry: a relation $\mathrm{R}$ is symmetric if
for all $S, S' \in \sS^{(k)}$, $S\mathrm{R}S'$ implies
$S'\mathrm{R}S$;
(3) transitivity:
a relation $\mathrm{R}$ is transitivity if
for all $S, S', S'' \in \sS^{(k)}$, $S\mathrm{R}S'$
and $S'\mathrm{R}S''$ implies $S\mathrm{R}S''$;
(4) isomorphic: a relation $\mathrm{R}$ is isomorphic if for all $S, S' \in \sS^{(k)}$,
$S\mathrm{R}S'$ implies that $S$ and $S'$ are attributed-isomorphic (i.e., $\rho(S)=\rho(S')$).
\end{definition}

Equivalence-isomorphic relations have the advantage of 
producing homogeneous classes w.r.t. patterns, where 
a class can be directly assigned to its pattern.
We study two instances of these relations in our experiments: 
(1) pattern percolation and (2) shared hubs (defined in Section~\ref{s:relations}). 

\section{Relation-based Graph Pattern Mining} \label{sec:frame}

This section presents our framework \name 
to compute and estimate $\alpha$ and $F$ for all subgraphs and patterns in the input graph $G$. 
The central idea of \name
is to reduce the number of relation tests ($S\mathrm{R}S'$) by employing sampling methods on a high-order network (HON) of $G$. 

The HON used in our framework, also denoted by $G^{(k)}$,  
has its nodes composed by $k$-node \CISes in $G$, where
 two \CISes have an edge if they share $k-1$ nodes (Def. ~\ref{def:hon}).
\name builds $G^{(k)}$ on the fly for subgraph sampling using MCMC output sampling~\cite{AlHasan:2009,Wang:2014,Guise:2014}. 
This MCMC process is performed through a \emph{random walk} (RW) over $G^{(k)}$, keeping in memory
only a single $k$-HON neighborhood at a time (Def. \ref{def:neigh}).

\begin{definition}[$k$-HON of $G$, or $G^{(k)}$] \label{def:hon}
 A $k$-HON $G^{(k)} = (\sS^{(k)}, E^{(k)})$ 
is a network where 
 $\sS^{(k)}$ composes the set of nodes and 
$E^{(k)}$ represents the set of edges in $G^{(k)}$.  
More specifically, $E^{(k)} = \{(S,S') | 
 S \in \sS^{(k)}, S' \in \sS^{(k)}$ and $|V(S) \cap V(S')| = k-1 \}$.
\end{definition}
\begin{definition}[$k$-HON Neighborhood, $N^{(k)}(S)$]
 The $k$-HON neighborhood 
of a $k$-node subgraph $S$ in $G^{(k)}$ or, simply, $N^{(k)}(S)$,  
 is composed 
  by subgraphs that share k-1 nodes with $S$. Formally, 
 $N^{(k)}(S) = \{S' |  S' \in \sS^{(k)}   
\text{ and } |V(S) \cap V(S'))| = k-1\}$.
\label{def:neigh}
\end{definition}

However, %
existing \CIS sampling methods {\bf cannot}
estimate both $F$ and $\alpha$ (eqs. \eqref{eq:alpha} and \eqref{eq:Fgeneral}) since 
the asymptotic convergence of the estimate $F$ requires an exact computation of $\alpha$. 
Moreover, for some subgraph classes, an exact computation of 
$\alpha$ may be faster than MCMC sampling. 
\name combines an sampling technique with an exact computation algorithm 
achieving benefits from both.

\begin{figure}[ht!!!]
    \centering
    \includegraphics[width=3in,height=1.7in]{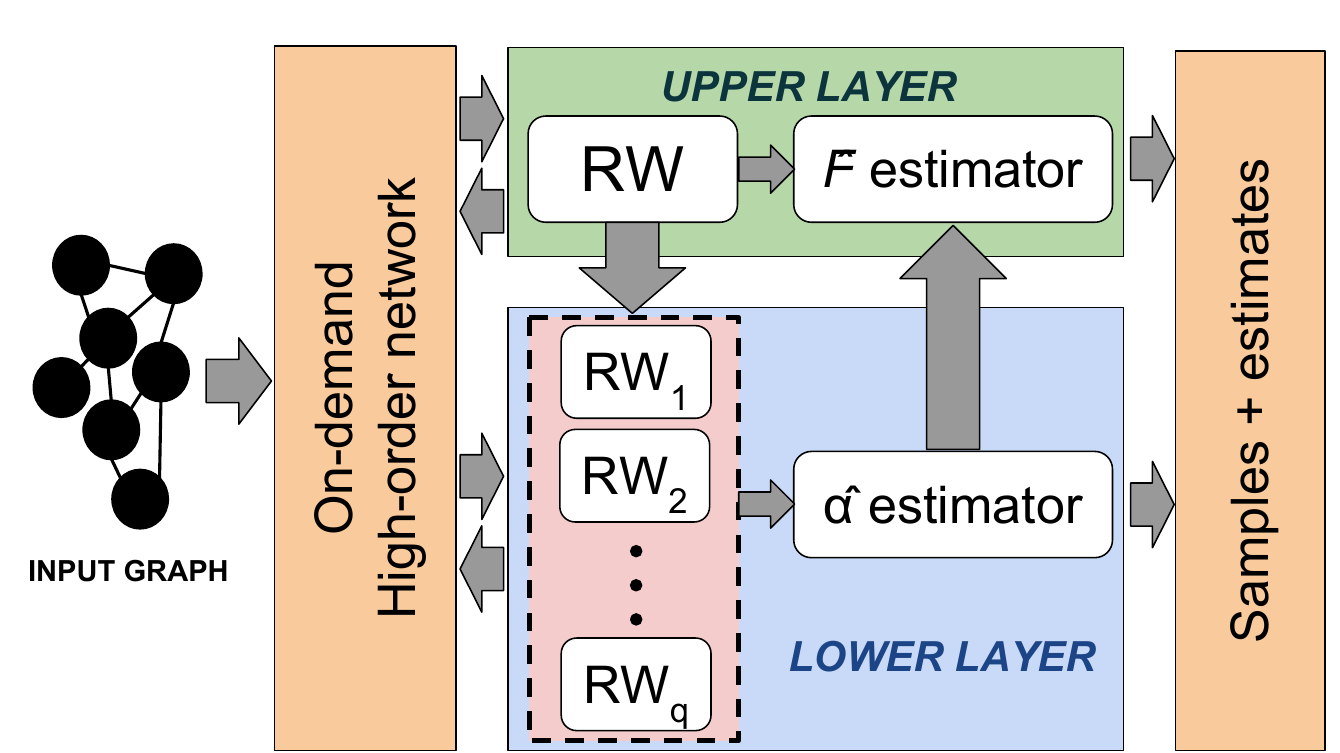}
    \caption{\name framework has two computation layers, the upper and the lower. They
 estimate $\alpha$ and $F$ (eqs. \eqref{eq:alpha} and \eqref{eq:Fgeneral}) in parallel, following a producer-consumer model.}
    \label{fig:frame}
\end{figure}

\noindent \textbf{Framework overview.}
Our parallel computing framework \name has two
 computation layers, denoted upper and lower.
The \emph{upper layer} is responsible for estimating $F$, 
while the \emph{lower} one estimates $\alpha$.
As shown in Figure~\ref{fig:frame},
\name receives a input graph and 
builds on-demand (i.e., when requested by a layer)
the HON $G^{(k)}$.
The upper layer peforms a RW  on $G^{(k)}$
 to sample subgraphs whose classes (and their $\alpha$'s)
will be computed. 

Let $S$ be a subgraph sampled in the upper layer.  
Then, $S$ is  sent 
to the lower layer, which produces a 
{\em finite-sample unbiased} estimate 
of $\alpha(\mathrm{R},S)$ (eq. \eqref{eq:alpha})
from $q$ independent random walks processes on $G^{(k)}$ (depicted 
in Fig. \ref{fig:frame} by the dotted lined block).
The estimated value of $\alpha$ 
is  returned to the upper layer in order to compute $F$ (eq. \eqref{eq:Fgeneral}),
with a consistent estimator.
At the end, the subgraphs sampled by \name and 
the all estimates of $\alpha$ and $F$ are given to the user.

\name parallelizes the tasks from the two computation layers
following a \emph{producer-consumer} model. 
This model fits naturally in our mining process since
the producer and consumer roles can be mapped directly to 
the upper and lower layers, respectively.
In addition, as the RWs procedures in the lower layer 
are independent, they are also computed simultaneously by 
\name.

\subsection{Computing $\mathbf{\alpha}$ (lower layer)} 
\label{s:alphaestimate}

In the lower layer, \name receives a subgraph and 
 it runs an exact computation of $\alpha$, 
limited to a computational budget $B$ w.r.t. the number of subgraphs that our 
framework can generate.
Once budget $B$ is exhausted, \name switches to a 
sampling procedure for estimating $\alpha$. 
One of our main contributions is an unbiased estimator of $\alpha$ 
that leverages the output of the exact computation algorithm in order to 
improve our estimator in terms of both accuracy and efficiency.

\subsubsection{Iteration-bounded exact computation} 
The iteration-bounded exact method is detailed in Algorithm~\ref{alg:exactbfs}, where
we traverse $G^{(k)}$ with a BSF algorithm. Starting from a subgraph $S$, given as parameter, 
our exact method iterates (or visit) on up to $B$ \CISes.
Foremost, 
it starts by setting $\alpha$, $C_{\mathrm{R},S}$, $H$ and $Q$
to their initial values (lines 1-4). 
While there are subgraphs to be visited in the queue $Q$,
we search for unvisited subgraphs (lines 7-15).
If the subgraph under analysis
($S''$) is valid according $h(\cdot)$,
 we add it in the set of visited subgraphs $H$ and in queue $Q$ for 
 further inspection (lines 9-10). 
 Note that $h(\cdot)$ is an optional 
 function that one may define in order to prune the search space 
 and, if $h(\cdot)$ is not given, all 
subgraphs are valid by default. 
The ideal (faster) scenario emerges when $h(\cdot)$ 
restricts the subgraph exploration to only subgraphs in $C_{\mathrm{R},S}$. 
Moreover,  
 $\alpha$ and $C_{\mathrm{R},S}$ are updated if $S''$
is R-related to subgraph $S$ (12-13).
The algorithm returns 
in two cases: (1) if the budget 
for traversing $G^{(k)}$ is finished and $C_{\mathrm{R},S}$ 
is incomplete (lines 14-15) or (2) 
when there is no unchecked subgraph
 in $Q$ and $C_{\mathrm{R},S}$ was completely generated (line 16).

In worst case, the time complexity of Algorithm \ref{alg:exactbfs} is
$O(\min(B,|\sS^{(k)}|)^2)$ when $G^{(k)}$ is dense.
In addition,
the amount of memory necessary to run 
 the exact method is bounded by $O(\min(B,|\sS^{(k)}|) + |V(G)|)$, 
where $|V(G)|$ is the space required to keep $N^{(k)}(S)$ of a certain subgraph $S$ 
for small values of $k$ ($k<10$).

\begin{algorithm}[t!]
\caption{Iteration-bounded Class Computation} \label{alg:exactbfs}
\begin{small}
        \SetAlgoLined
        \SetKwInOut{Input}{input}\SetKwInOut{Output}{output}
        \SetKwRepeat{Do}{do}{while}
        \Input{$G$, a input graph}
        \Input{$S$, a $k$-node subgraph}
        \Input{R, a user-defined subgraph relation}
        \Input{$B$, the maximum number of steps allowed}
        \Input{$g(\cdot)$, a user-defined function to compute $\alpha(C)$ }
        \Input{$h(\cdot)$, a user-defined function to restrict $G^{(k)}$ (optional)}
        \Output{$\alpha$, the weight value of $S$'s class}
        \Output{$C_{\mathrm{R},S}$, the class of $S$ with relation R}
        \Output{$c$, a boolean to say if $C_{\mathrm{R},S}$ is complete.}

\BlankLine
		$\alpha \leftarrow 0$\;
                $C_{\mathrm{R},S}\leftarrow \emptyset$\;
                $H \leftarrow \{ S \} $\;
                $Q$.push($S$)\tcp*{queue of subgraphs} 
                
                \tcc{Verify all non-visited subgraphs.}
                \While { $Q \neq \emptyset$ } {
                	$S' \leftarrow Q$.pop()\;
                    
                    	\For { $S'' \in N^{(k)}(S')) \setminus H$} {
                        	\tcc{Verify if $S''$ must be visited.}
                        	\If {$h(S'') = true$} {
                        		$H \leftarrow H \cup \{ S'' \}$\;
                        	        $Q$.push($S''$)\;
               				\tcc{Update values if related.}
               				\If {$S''\mathrm{R}S = true$} {
	                       			$\alpha \leftarrow \alpha + g(S'')$\;
        	               			$C_{\mathrm{R},S} \leftarrow C_{\mathrm{R},S} \cup \{S''\}$\;
                			}
					\If{$|H| = B$} {\Return{$\alpha$,$C_{\mathrm{R},S}$, \text{false}}\;}
				}
                	}                	
                }
                \Return{$\alpha$,$C_{\mathrm{R},S}$, \text{true}} 
\end{small}
\end{algorithm}

\subsubsection{Estimating $\alpha$ via sampling}
We propose to estimate $\alpha(\mathrm{R}, S)$ 
by performing \emph{random walk tours} (RWTs) on $G^{(k)}$.
A RWT on $G$ is a special type of random walk
that considers the first-return time of the RW Markov chain (i.e., the first time the RW returns to the starting node) to estimate network statistics from sampled nodes~\cite{cooper2016fast, Avrachenkov:2016}.
To the best of our knowledge, 
this is the first work that applies RWTs on high-order networks which is 
 particularly challenge due to the large number of nodes 
 in $G^{(k)}$ ($k$-node subgraphs).
Next, we introduce a variety of innovations and optimizations to deal with this issue.

\begin{itemize}[noitemsep,nolistsep,leftmargin=*]
\item {\it Reusing the exact computation to speed-up return times:} 
The return time (steps) of a RW to the same initial state can already be prohibitively on $G^{(k)}$. %
To speed-up return times, we transform the subgraphs retrieved in an iteration-bounded exact computation of $\alpha$, 
specifically the ones that belong to the class $C_{\mathrm{R},S}$, in a collapsed supernode.
The supernode is a {\em virtual} node of $G^{(k)}$ with a massive number of edges, which speeds-up return times, as the expected return time is inversely proportional to the number of edges of a node~\cite{Avrachenkov:2016}.
We call this collapsed  node a supernode and denote it $\mathcal{I}_S$, and it often refers to it as both as a set and as a {\em collapsed node}. 
\item {\it Non-backtracking random walks:} 
On the other hand, 
very short return times can point to 
a poor mixing of the random walk, where the massive degree of supernode $\mathcal{I}_S$ impedes the exploration of $G^{(k)}$ by the walkers.
To overcome this challenge, we use \emph{non-backtracking random walks} (nRWs). 
The advantage of nRWs over the ordinary RWs is that it avoids resampling recently sampled nodes~\cite{alon2007non,lee2012beyond}, 
ensuring return times greater than two. We denote this approach a \emph{non-backtracking random walk tour} (nRWT).
\end{itemize}

\noindent
{\bf The $\mathbf{\hat{\alpha}}$ estimator:}
Theorem~\ref{thm:unbias} shows an unbiased estimator of $\alpha$ using nRWTs.

\begin{theorem} \label{thm:unbias}
Let $G$ be the input graph,
$\mathrm{R}$ be a user-defined relation and
$g(\cdot)$ a user-defined function. 
Moreover, consider $S$ a subgraph in $\sS^{(k)}$,
$\mathcal{I}_S$ (supernode)
a set of subgraphs related to $S$  
and $\mathcal{T}^r = \{S_1^r, \dots S_{T}^r\}$ 
the sample path of subgraphs visited by the $r$-th RWT on $G^{(k)}$, 
where  $1<r<q$ and
$S^r_i$ is the subgraph reached by the RWT in step $i$.
Because $\mathcal{T}$ is a RWT, $S_1\subseteq \mathcal{I}_S, S_T \subseteq \mathcal{I}_S$ and $S_j \notin \mathcal{I}_S$, $\forall 1 < j < T$. 
Then 
\begin{equation}\label{eq:alphahat}
\begin{aligned}
& \hat{\alpha}(q,\mathrm{R},S) = \\
&  \frac{\sum_{S'\in \mathcal{I}_S} |N^{(k)}(S')\setminus \mathcal{I}_S|}{q} \sum_{\mathclap{r=1}}^{q} \sum_{i = 2}^{|\mathcal{T}^r|-1} 
 \frac{g(S^r_i)  \cdot {\bf 1}_{\{S^r_i\mathrm{R}S\}} }{|N^{(k)}(S^r_i)|} \\
 & + \sum_{\mathclap{S' \in \mathcal{I}_S}} g(S')  \cdot {\bf 1}_{\{S'\mathrm{R}S\}} 
\end{aligned}
\end{equation}

is an unbiased estimator of  $\alpha(\mathrm{R},S)$.
\end{theorem}%
The proof of Theorem~\ref{thm:unbias} is broken down into multiple parts.
First, we formalize the non-backtracking RW on $G^{(k)}$, and show its steady state is the same as the steady state of a standard RW.
\begin{definition}[nRW on $G^{(k)}$] \label{def:nrw}
A nRW  on $G^{(k)}$ is a 2nd order Markov chain with transition
probability matrix:
\begin{equation}
  \begin{aligned}
 \mathbf{T}(S | S',S'') =
 \begin{cases} 
     0 &  \text{, if } S = S'',  |N^{(k)}(S')| > 1 \\
     1 &  \text{, if } S = S'',  |N^{(k)}(S')| = 1 \\
     \frac{1}{|N^{(k)}(S')|-1} & \text{, otherwise},
 \end{cases}
   \end{aligned}
\end{equation}
   where $S,S',S''$ are $k$-node \CISes in $\sS^{(k)}$.
\end{definition}
\begin{lemma} \label{l:SS}
The steady state distribution $\pi$ of the nRW 2nd order Markov chain given by the transition probability matrix $\mathbf{T}$ of Definition \ref{def:nrw} is:
\begin{equation}
  \begin{aligned}
	\pi(S) = \frac{|N^{(k)}(S)|}{\sum_{S' \in \sS^{(k)}} |N^{(k)}(S')|},
   \end{aligned}
\end{equation}
where $|N^{(k)}(S)|$ is the number of neighbors of $S$ in the $k$-HON $G^{(k)}$, as 
described in Definition \ref{def:neigh}.
\end{lemma}
The proof of Lemma~\ref{l:SS} comes from the fact that the 
resulting 2nd order Markov chain is irreducible, aperiodic, and Harris recurrent~\cite{Lovasz1996}, as long as a random walk on $G^{(k)}$ provides such properties.
As $G^{(k)}$ is finite, the chain is trivially Harris recurrent.
Moreover, a Markov chain on $G^{(k)}$ is irreducible and aperiodic since  $G^{(k)}$ is a strongly connected and non-bipartite graph~\cite{Ribeiro2010}, which is shown next.
\begin{lemma}\label{l:Conn}
$G^{(k)}$ is a strongly connected and non-bipartite graph $\forall k \in \{ i | 2 \leq i \leq m-1 \}$
 if (a) $G$  is a strongly connected component and (b)
$G$ has a cycle with $m$ nodes and $m$ is odd.
\end{lemma}
\begin{proof}
Let $L = \{v_1, \dots, v_m\}$ be the set of nodes that composes a cycle in a strongly connected graph $G$, where
 $|L| = m$, $v_1$ is connected to $v_m$, $v_i$ is connected to $v_{i+1}$
and $m$ is a odd number. As $L$ is connected by definition, we may enumerate \CISes from their nodes. 
Let $L^{(k)}$ be the set of $k$-node subgraphs using \emph{only} the nodes in $L$.
Then,  $L^{(k)}$
has exactly $m$ \CISes, $\forall k \in \{ i | 2 \leq i \leq m-1 \}$. 

We now need to show that these $m$ $k$-node subgraphs generate a cycle in  $G^{(k)}$.
Let  $L^{(k)} = \{S_1, \dots, S_m\}$ where $V(S_i) = \{v_i, \dots, v_{w}\}$
 and $w=i+k-1 \mod m$. For instance, considering $k = 2$ and $m = 3$, we have
$V(S_1) = \{ v_1,v_2 \}$ and $V(S_3) = \{ v_3,v_1\}$.
These $k$-node subgraphs are connected in $G^{(k)}$
whenever they share $k-1$ nodes between them, according to Definition~\ref{def:hon}, i.e.,
$S_i$ is connected to $S_{i-1}$ and $S_{i+1}$ in $G^{(k)}$, $\forall i \in \{ 2 \leq i \leq m-1 \}$.
\end{proof}
Now that we know how nRWs have the same steady state distribution as RWs, the proof concludes with Lemma~\ref{lem:unbias}, showing that we can construct an unbiased estimator of $\alpha$ from the sample path of each nRWT.

\begin{lemma}\label{lem:unbias}
	Let $\mathcal{T}^r$ be the set of sampled nodes in the $r^{th}$ nRWT on $G^{(k)'}$, $r \geq 1$ starting at the supernode $\mathcal{I}_S$. Then,
	\begin{align}\label{eq:unbias}
	\begin{array}{l}
	\mathbb{E} \big[ \sum_{S'\in \mathcal{I}_S} |N^{(k)}(S')\setminus \mathcal{I}_S| \sum_{i=2}^{|\mathcal{T}^r|-1} \frac{g(S^r_i)  \cdot {\bf 1}_{\{S^r_i\mathrm{R}S\}}  }{|N^{(k)}(S^r_i)|} \\ + \sum_{S^\dagger \in \mathcal{I_S}} g(S^\dagger)\cdot {\bf 1}_{\{S^\dagger\mathrm{R}S\}} \big]	
	 = \alpha(\mathrm{R},\mathcal{I}_S)
	\end{array}
	\end{align}
	  
\end{lemma}

\begin{proof}
	
	We start by rewriting Eq. \ref{eq:unbias} as
	
	\begin{align}\label{eq:unbias-new}
	   \begin{array}{l}
	|N^{(k)}(S')\setminus \mathcal{I}_S| \mathbb{E} \big[ \sum_{S'\in \mathcal{I}_S} \sum_{i=2}^{|\mathcal{T}^r|-1} \frac{g(S^r_i)  \cdot {\bf 1}_{\{S^r_i\mathrm{R}S\}}  }{|N^{(k)}(S^r_i)|} \big]	\\
	 = \alpha(\mathrm{R},\mathcal{I}_S) - \sum_{S^\dagger \in \mathcal{I_S}} g(S^\dagger)\cdot {\bf 1}_{\{S^\dagger\mathrm{R}S\}}
	\end{array}
	\end{align}
 	
	With the nRW starting at $\mathcal{I}_S$, we can rewrite the expected value on the left hand side of eq.~\eqref{eq:unbias-new} as
	
	\begin{equation}\label{eq:unbias2}
	\begin{array}{l}
	\mathbb{E}\big[\sum_{i=2}^{|\mathcal{T}^r|-1} \frac{g(S^r_i)  \cdot {\bf 1}_{\{S^r_i\mathrm{R}S\}}  }{|N^{(k)}(S^r_i)|} \big]  =
	\sum_{\mathcal{S}^{(k)} \setminus \mathcal{I}_S } \mathbb{E}\big[\frac{ \mathbb{S}_i^r g(S^r_i)  \cdot {\bf 1}_{\{S^r_i\mathrm{R}S\}}}{|N^{(k)}(S^r_i)|} \big],
	\end{array}
	\end{equation}
	
	\noindent where $\mathbb{S}_i^r$ is the number of times the Markov Chain reaches $S_i$ in the $r^{th}$ tour.
	
	Given a renewal-reward process with inter-reward time distributed as $|\mathcal{T}^r|$, $r \geq 1$ and reward $\mathbb{S}_i^r$, the renewal theorem gives us that
	
	$$\pi(S^r_i) = \mathbb{E}[|\mathcal{T}^r|]^{-1} \mathbb{E}[\mathbb{S}_i^r].$$
	
	\noindent Hence, \ref{eq:unbias2} becomes
	
	\begin{equation}
    \label{eq:unbias3}
    \begin{array}{l}
	 \sum_{i=2}^{|\mathcal{T}^r|-1} \mathbb{E}\big[\frac{ \mathbb{S}_i^r g(S^r_i)  \cdot {\bf 1}_{\{S^r_i\mathrm{R}S\}}}{|N^{(k)}(S^r_i)|}\big] = \\ \sum_{i=2}^{|\mathcal{T}^r|-1} \frac{ \mathbb{E}\big[|\mathcal{T}^r|\big] \pi(S^r_i)  g(S^r_i)  \cdot {\bf 1}_{\{S^r_i\mathrm{R}S\}}}{|N^{(k)}(S^r_i)|}
	\end{array}
	\end{equation}
	
	\noindent Moreover, it follows from Kac's theorem that $\mathbb{E}[|\mathcal{T}^r|] = \frac{1}{\pi(\mathcal{I}_{S})}$. Thus, by levareging Lemma \ref{l:SS} we can rewrite eq. \eqref{eq:unbias3} now as
	
	\begin{equation}\label{eq:unbias4}
	\begin{array}{l}
	\sum_{i=2}^{|\mathcal{T}^r|-1} \frac{ \mathbb{E}\big[|\mathcal{T}^r|\big] \pi(S^r_i)  g(S^r_i)  \cdot {\bf 1}_{\{S^r_i\mathrm{R}S\}}}{|N^{(k)}(S^r_i)|}
	\\ = \frac{1}{\sum_{S'\in \mathcal{I}_S} |N^{(k)}(S')\setminus \mathcal{I}_S|} \sum_{i=2}^{|\mathcal{T}^r|-1} g(S^r_i)  \cdot {\bf 1},
	\end{array}
	\end{equation}

	\noindent and see that since the tours only compute the function outside the supernode, Lemma \eqref{lem:unbias} follows directly from eq. \eqref{eq:unbias2} as we wanted to show.	
\end{proof}

\begin{proof}[Proof of Theorem \ref{thm:unbias} ]
By the Strong Markov Property, each tour is independent of the others. Thus, by Lemma \ref{lem:unbias} and linearity of expectation, Theorem \ref{thm:unbias} holds. 
\end{proof}
\subsection{Estimating $F$ (upper layer)}

The ordinary RW on $G^{(k)}$ can be seen as sampling of subgraph classes, 
where the class $C_{\mathrm{R},S}$ is sampled whenever
$S$ is visited by the MCMC process, for any relation $\mathrm{R}$ of interest.
Thus, we may estimate $F$ (eq.~\eqref{eq:Fgeneral}) using
a simple RW on $G^{(k)}$, with asymptotic bias 
given by the number of \CISes in the classes. 
Fortunately, such bias may be removed using the Horvitz-Thompson estimator~\cite{ribeiro2012estimation}. 
Theorem \ref{the:fapp} gives a consistent estimator of $F$ of eq.~\eqref{eq:Fgeneral}.
\begin{theorem}\label{the:fapp}
Let  $\mathcal{M}_t = \{S_i, \dots, S_t\}$ be 
\CISes  visited by the upper layer RW over $G^{(k)}$ after t > 1 steps.
Consider the estimator
\begin{equation}\label{eq:Fapp2}
\hat{F}(t, q, G, \mathrm{R}, P) = \frac{1}{\lambda} 
\sum_{\mathclap{S \in \mathcal{M}_t}}
\frac{\hat{\alpha_1}(q, \mathrm{R}, S)\cdot {\bf 1}_{\{\rho(S) = P\}}}{\hat{\alpha_2}(q,\mathrm{R},S)}, 
\end{equation} 
where $\hat{\alpha}_1$ estimates $\alpha(\mathrm{R},S)$, and
$\hat{\alpha_2}$ estimates $\sum_{S' \in C_{\mathrm{R},S}} |N^{(k)}(S')|$ (the steady state probability a CIS in $C_{\mathrm{R},S}$ is sampled by the upper layer random walk), $q \propto t$ is the number of tours used in the estimates of $\hat{\alpha}_1$ and $\hat{\alpha}_2$.
and  $
\lambda = \sum_{{S \in \mathcal{M}_t}}\frac{\hat{\alpha_1}(q,\mathrm{R},S)}{\hat{\alpha_2}(q,\mathrm{R},S)}.$
Then, $\hat{F}$ is consistent (asymptotically unbiased), that is,
\begin{equation*}
\lim_{t \to \infty} \hat{F}(t,q,G,\mathrm{R}, P) \overset{a.s.}{=} F(G, \mathrm{R}, P).
\end{equation*}
\end{theorem}
\begin{proof}[Proof (sketch)]
The main challenge is to prove that the ratio $\hat{\alpha}_1/\hat{\alpha}_2$ in eq.~\eqref{eq:Fapp2} does not create issues in the asymptotic convergence.
From Lemma~\ref{lem:unbias} we know that $\hat{\alpha}_1$ and $\hat{\alpha}_2$ are both unbiased and i.i.d. for any number of tours $q$.
And we can make an infinite number of tours $q \propto t \to\infty$ because tours are finite a.s.~\cite{Avrachenkov:2016}. Averaging $\hat{\alpha}_1$ and $\hat{\alpha}_2$ over all such tours gives a.s.\ convergence by the Strong Law of Large Numbers (SLLN). 
This means that we can directly use the SLLN for Markov chains~\cite{Ribeiro2010} to establish a.s.\ convergence over the entire estimator.
\end{proof}

\section{Experimental Evaluation}
\label{sec:results}
In this section we evaluate the proposed method 
with real-world networks in two different tasks.
First, we showcase two applications of our approach.
Second, we empirically validate the accuracy of our estimator.

\noindent %
\textbf{Hardware.}
We evaluate \name in the XSEDE Jetstream~\cite{Jetstream1} servers with Intel Xeon E5-2670 CPUs with a
total of 44 execution threads at 2.5GHz per core and 120GB RAM.

\noindent 
\textbf{Datasets.}
We used the following datasets in our evaluation.

\textit{Yeast:} the Yeast protein-protein interaction dataset~\cite{bu2003topological} 
    is a graph, where proteins are nodes, 
an edge means an %
    interaction and
a node
    attribute gives the protein function.
The largest component was
    considered and self loops were excluded, totaling 2224 nodes, 6609 edges and 13 labels. 

\textit{DBLP-$n$:} DBLP~\cite{sun2011co} is a bibliographic temporal network with three node types (authors, venues, and topics) which are connected by the publication of a paper. We define DBLP-$n$, as the DBLP snapshot given by publications between the years $[2001+2n,2001+2n+1]$, for $n=1,2,3$. There are 22412 authors in all 3 networks, with up to 127851 edges.

\textit{Microsoft scientific research (MSR):}  This is a bibliometric network from the Microsoft scientific resesarch dataset~\cite{tang2008arnetminer}. It contains 
authors and their collaborations considering only the data mining community\footnote{ICDE, ICDM, KDD, PAKDD, PKDD, RECSYS, SDM, TKDD, TKDE, VLDB and WSDM.}.
    The largest connected component has 26855
    nodes and 101320 edges. 
We considered the year of
    her first publication to determine the node attribute:
(1) junior, (2) intermediate and (3) senior researcher.

Note that, although the input graphs have 
a relatively small number of nodes (up to 27k), 
their HONs are huge with up to 66M of nodes (subgraphs).

\noindent
{\bf Subgraph relations.} \label{s:relations}
Our experiments consider two subgraph relations.
The first subgraph relation, PERC (Definition \ref{def:perc}), 
generalizes the $k$-clique percolation method~\cite{PalEtAl05}.
Here, a {\em subgraph class} arises from percolating a specific subgraph pattern over the graph. 
Note that our definition uses a general pattern $P$, i.e., it is not restricted to cliques. 
However, %
sparse patterns should be avoided, specially in non-labeled graphs, since 
they may percolate across the whole network, losing meaning. 

The second relation is \emph{shared $d$-hubs} (SH$_d$) (Definition \ref{def:shdn}), 
 where related subgraphs must also share
the nodes whose degrees are larger than a specified threshold $d$.
In fact, SH$_d$  
is particularly interesting when applied on scale-free networks since 
it may be used to reduce the bias 
induced by high-degree nodes in GPM methods.

\begin{definition}[Pattern Percolation, PERC]\label{def:perc}
A subgraph $S$ is related to (or percolates)  $S'$ iff:
 (a) $S = S'$ or
(b) $|V(S) \cap V(S') = k-1$ and 
$\rho(S)=\rho(S')$ 
or  
(c) $S$ percolates $S''$ and $S'$ percolates $S''$.
\end{definition}
\begin{definition}[Shared $d$-Hubs, SH$_d$]\label{def:shdn}
Let $N(v_i)$ be the set of neighbors of the node $v_i$ in 
the input graph $G$ and
$H(S) = \{ v_i \in V(S) | |N(v_i)| \geq d \}$ 
be the set of high degree nodes (hubs) of a subgraph $S$ 
given a threshold $d$. 
Then, two subgraphs $S$, $S'$ 
are related (or equivalent) iff:
(a) $S = S'$ or
(b) $H(S) = H(S')$ and $\rho(S)=\rho(S')$.
\end{definition}

\subsection{Application: motif class counting}\label{s:motifs}
Motif class counting (MCC) is a generalization of the standard motif counting (SMC) problem~\cite{prvzulj2007biological}, where
the frequency of a pattern $P$ is given by 
the number of {\em subgraph classes} a pattern $P$ has in
 the input graph $G$ 
 rather than the total number of matchings of $p$ in $G$.
The rationale behind MCC is  
that subgraphs belonging to a same class 
are equivalent and, then, they should not be counted twice.
The SMC problem can be easily 
 mapped to MCC by setting  $1_{\{S\mathrm{R}S'\}}=1$ iff $S = S'$ and $g(\cdot)=1$. 

We apply MCC with PERC relation on the Yeast graph, 
where the relevance score of a pattern is measured by the proportion of classes it has in $G$, 
obtained from $F$ (eq.~\eqref{eq:Fgeneral}) with $g(S)=\frac{1}{C_{\mathrm{R},S}}$ in eq.~\eqref{eq:alpha}.
We consider only 4-node subgraphs which are quasi-cliques~\cite{liu2008effective} with
density greater than 0.5, i.e., each node is connected to at least other 2 nodes in the subgraph.

In this experiment, we
 consider 4-node subgraphs in quasi-cliques~\cite{liu2008effective} with
density is greater than 0.5, i.e., each node is connected to at least other 2 nodes.
 
To show that the frequent patterns in MCC are different from SMC,
we compare the topology of the top frequent motifs obtained
 by the exact method and MCC in the Yeast graph. 
 More specifically,
  we compute the Kendall's Tau
  correlation between two motifs' rankings:
  one returned by the exact and traditional method and the second ranking from MCC-PERC (converged). 
   The result is given in  Table \ref{tab:rank}, where in consider different ranking' sizes. 
Indeed, we can't reject the hypothesis of the rankings being independent, showing empirically the difference between 
them.
   
\begin{table}[h]
\centering
\begin{tabular}{r|r|r}
\toprule
\textbf{Ranking Size}  & \textbf{Test statistic} & \textbf{p-value}    \\ \midrule
50           & 0.01        & 0.90    \\
100          & 0.00       & 0.95    \\
500          & 0.04         & 0.17    \\
\bottomrule
\end{tabular}
\caption{Comparison between two rankings of patterns returned by SMC and MCC-PERC in Yeast graph through 
Kendall Tau correlation. We reject the hypothesis of the rankings being independent as we increase the ranking' size.}
\label{tab:rank}
\end{table}
	
A closer look also reveals fundamental structural differences between the SMC method 
and MCC. 
Table~\ref{tab:topmotifs} shows topologies 
-- colors represent protein functions -- and relative frequencies ($F$) of the
 {\em top 3} motifs listed by the two approaches evaluated.
  Note that the patterns found by MCC are structurally different, having a greater diversity of
 protein functions (node attributes).

\begin{table}[h!]
\centering
\begin{tabular}
{|c|c@{}c@{}|c@{}c@{}|}
\toprule
\textbf{Rank}  & \multicolumn{2}{|c|}{\textbf{SMC}} & \multicolumn{2}{|c|}{\textbf{MCC-PERC}} \\ 
\midrule
	& $F$ & Motif & $\hat{F}$ & Motif \\ %
1 &  4.15 &      
\begin{minipage}{.07\textwidth}
\includegraphics[scale=0.08]{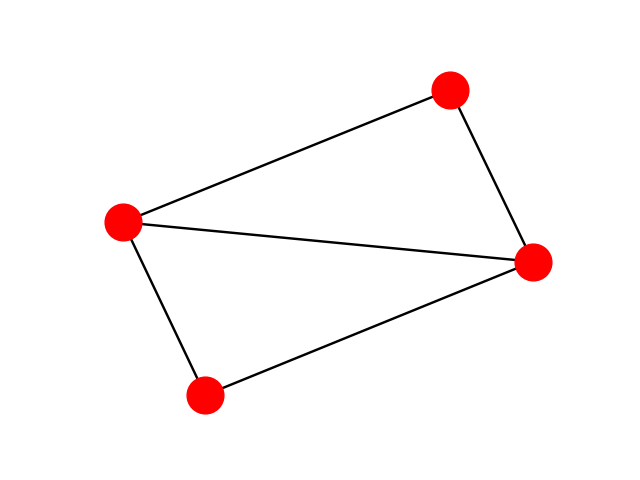} 
\end{minipage}
& 0.58   &   
\begin{minipage}{.07\textwidth}
\includegraphics[scale=0.08]{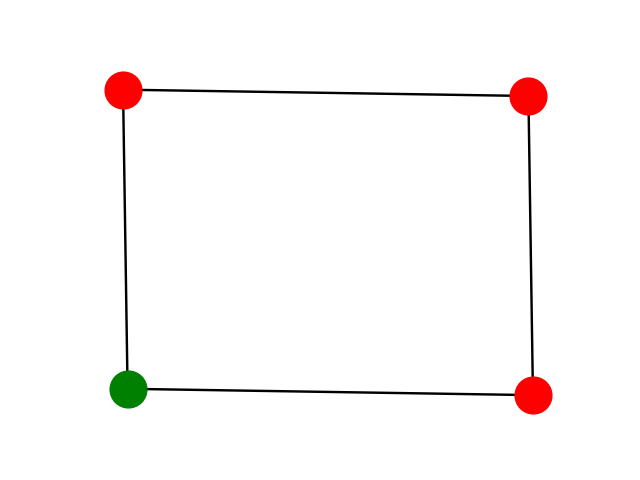} 
\end{minipage}
\\ \midrule
2 &  3.66 &  
\begin{minipage}{.07\textwidth} 
\includegraphics[scale=0.08]{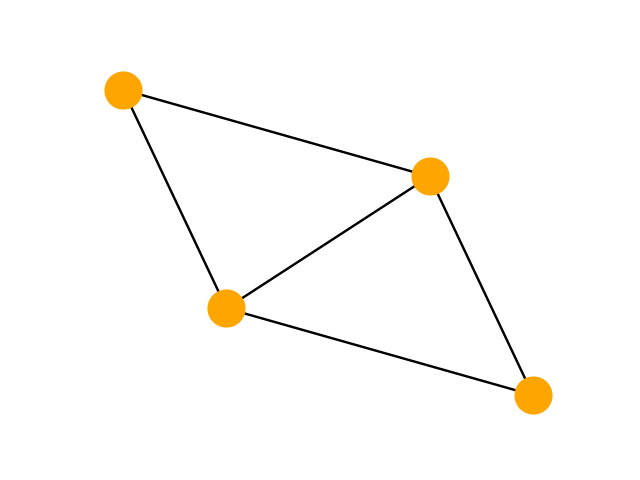}
\end{minipage} 
& 0.56   & 
\begin{minipage}{.07\textwidth} 
\includegraphics[scale=0.08]{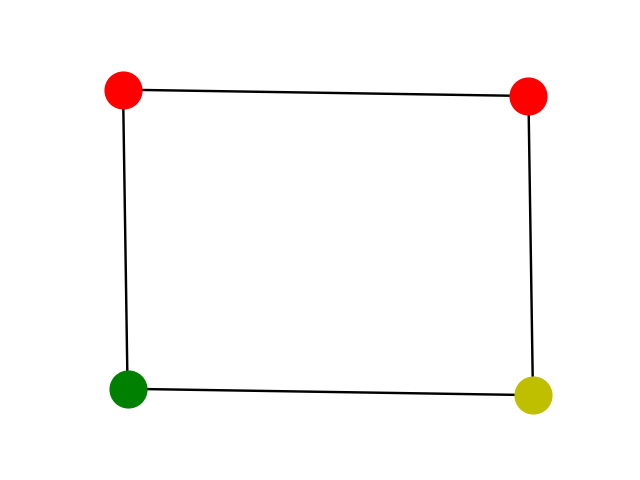} 
\end{minipage}
\\ \midrule
3 &   2.02 & 
\begin{minipage}{.07\textwidth} 
\includegraphics[scale=0.08]{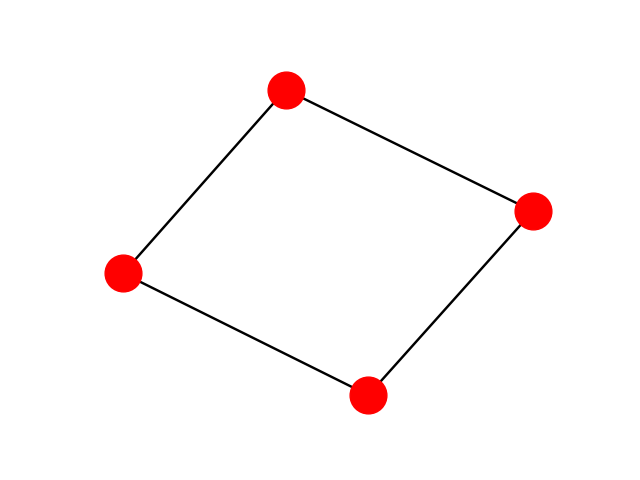}
\end{minipage}
 &  0.41  &  
\begin{minipage}{.07\textwidth} 
\includegraphics[scale=0.08]{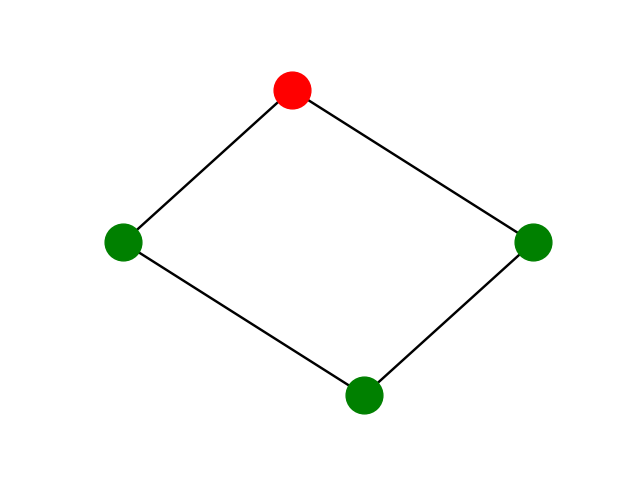} 
\end{minipage}
\\ 
\bottomrule
\end{tabular}
\caption{Comparison between the top 3 most frequent motifs obtained using SMC and MCC-PERC methods in Yeast graph. Indeed, the methods can find different relevant patterns. 
}
\label{tab:topmotifs}
\end{table}

\subsection{Application: Subgraph Prediction}
We now focus on the flexibility of our generalization of GPM by using it to train a deep neural network model that predicts subgraph dynamics on temporal graphs.
For this task, we use the Subgraph Pattern Neural Network (SPNN) model of Meng et al.~\cite{Meng2018}.
Consider a temporal graph $(G_n)_{n=1}^3$.  
Meng et al.\ predicts how CISes (connected induced subgraphs) on $G_1$ will evolve in $G_2$. 
More precisely, the goal of the model is to minimize the negative log-likelihood loss
\begin{equation}
\label{eq:Loss}
\mathcal{L}(G_1; \bW) = - \sum_{C \in \sS_1^{(k)}} \log \text{Pr}(y(S) | S, \bW),
\end{equation}
where $\bW$ are the neural network parameters, $\sS_1^{(k)}$ are the $k$-node CISes of $G_1$ and $y(S)$ is a class label indicating which subgraph $S$ has evolved into on $G_2$. For instance, for $k=3$, $S$ is a triangle or a vee on $G_1$ and $y(S) \in \{0,1\}$ indicates whether $S$ becomes disconnected or not on $G_2$.
Once $\bW$ learned by minimizing eq.~\eqref{eq:Loss}, the model is applied to subgraphs on $G_2$, predicting their evolution (labels) on $G_3$.	
To scalably optimize eq.~\eqref{eq:Loss}, we need to to train the model with stochastic gradient descent, which samples CISes on $G_1$.

Considering that the degree distribution 
of real networks often follows a power law, 
we soon realize that the high-degree hubs on $G_1$ induce a large number of \CISes, which will have a disproportional influence over the objective function in eq.~\eqref{eq:Loss}.
Using \name, practitioners now can reduce this influence by defining a subgraph relation in an alternative loss function to eq.~\eqref{eq:Loss}:
\begin{equation}
\label{eq:LossAlt}
\mathcal{L}'(G_1; \bW) = - \sum_{S \in \sS_1^{(k)}} \alpha(R,S) \log \text{Pr}(y(S) | S, \bW),
\end{equation}
where $\alpha(R,S)$ is a weight of CIS $S$ defined by relation R. 
Using \name, we can efficiently estimate the required gradients of eq.~\eqref{eq:LossAlt}, $\sum_{S \in \sS_1^{(k)}} \alpha(R,S) \nabla_\bW \log \text{Pr}(y(S) | S, \bW)$, since our approach samples $S$ from $\sS_1^{(k)}$ with an asymptotically known bias and $\alpha(R,S)$ is estimated with a finite-sample unbiased estimator. The result is asymptotically unbiased estimates of the gradient that we use in what is known as the Markovian dynamic case of the Robbins-Monro stochastic optimization~\cite{delyon1996general}.

\noindent
{\em Experiment:} We consider the same task defined in Meng et al.\ for the DBLP dataset. Given an author, a venue and a topic, we want to predict whether the author will publish in this venue and in this topic in the next timestep. We consider DBLP as a sequence of 3 graphs (with a span of two years each).
The subgraph relation R is SH$_{100}$. We use 6000 samples for training (evolution from $G_1$ to $G_2$) and 2500 for testing (evolution from $G_2$ to $G_3$).
We use the same hyperparameters as in Meng et al.: 30\% of our training data to perform early stopping, the maximum number of epochs is 6000, learning rate is set to $0.01$, and L2 regularization strength is $0.001$.

\noindent
{\em Results:}
To evaluate the training of SPNN with the estimated class weights $\hat{\alpha}$ in eq.~\eqref{eq:LossAlt}, in the test phase, i.e., predicting how CISes in $G_2$ evolve on $G_3$, we compute the weighted accuracy $$W_{Acc} = \frac{\sum_{S \in \sS_2^{(k)} }\alpha(R,S) {\bf 1}_{\{\hat{y}(S;\bW) = y(S)\}}}{\sum_{S \in \sS_2^{(k)} }\alpha(R,S)},$$ 
where $\hat{y}(S;\bW)$ is the predicted evolution given by the trained model, $\sS_2^{(k)}$ are the $k$-node CISes of $G_2$. 
As $\alpha(R,S)$ is expensive to compute, we estimate $W_{Acc}$ we use the approximation $\alpha(R,S) \approx \hat{\alpha}(100,R,S)$.

We used \name to train the
model against assuming all CISes have the same weight, i.e., $\alpha(R,S) = 1$,
$ \forall S \in \sS_1^{(k)}$, 
i.e., we are biasing against too many CISes that share the same high degree node, compared to the baseline of the original Meng et al.\ method. 
We used $q=100$ for both training and validation.
The resulting average weighted accuracies of an SPNN trained with the inverse bias w.r.t.\ the class sizes of relation SH$_{100}$ (eq.~\eqref{eq:LossAlt}) are 0.64 and 0.69 for the baseline and our method, respectively.
Note that training considering the classes sizes estimations improves our
performance in a scenario where predictions are made per class. We made a
paired t-test over the multiple runs resulting in a p-value of 0.03773, which
we use to argue that our approach has better performance than the baseline.

\subsection{Evaluating the accuracy of the $\hat{F}$ estimator}

\begin{table}[h!]
\centering
\begin{tabular}{c|c|c|c|c|c|c}
\toprule
\textbf{\#Exp.} & \textbf{Graph} & \textbf{Relation} & $\mathbf{k}$& $\mathbf{|I|}$ & \textbf{\#tours} & \textbf{\#samples}\\
\midrule
1 & Yeast	& PERC& 4 &100 &	1k &	100k\\
2 & MSR	& PERC& 4 &10k &	1k &	100k\\
3 & DBLP1	& SHs& 3 & 10k &	1k &	100k\\
\bottomrule
\end{tabular}
\caption{Experiments executed for weight-based motifs counting: datasets, subgraph relations and parameters.}
\label{tab:exp-conv}
\end{table}

Here, we evaluate the quality of our estimator $\hat{F}$ (eq. \eqref{eq:Fapp2}), 
applying MCC (see Sec. \ref{s:motifs}) problem
with two subgraph relations: PERC and SH$_d$ 
(Definitions~\ref{def:perc} and~\ref{def:shdn}, respectively) 
on three datasets (Yeast, DBLP1, and MSR). 
We focused on the three experiments which is described in 
Table \ref{tab:exp-conv}. For Yeast and MSR networks, we consider
 quasi-cliques with 4 nodes with density greater than 0.5 (i.e., each node must be connected to at least two others),  
while all 3-node subgraphs are processed in the DBLP1 graph.
In all cases, the pattern score estimated by $\hat{F}$ is computed using $g(S)=\frac{1}{C_{\mathrm{R},S}}$.

\begin{figure}[h!]
    \centering
    \subfloat[SSE analysis]{\label{fig:sseyeast}
    \includegraphics[width=1.6in, height=1.22in]{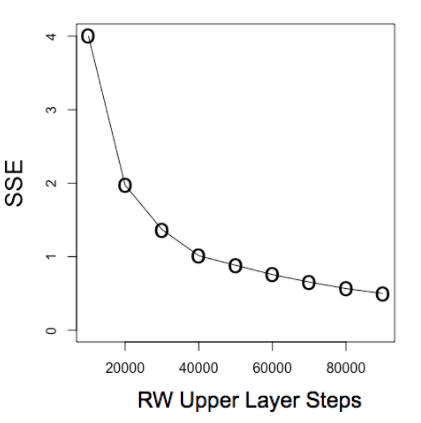}
}
    \subfloat[Top 5 convergence]{\label{fig:baryeast}
    \includegraphics[width=1.9in, height=1.22in]{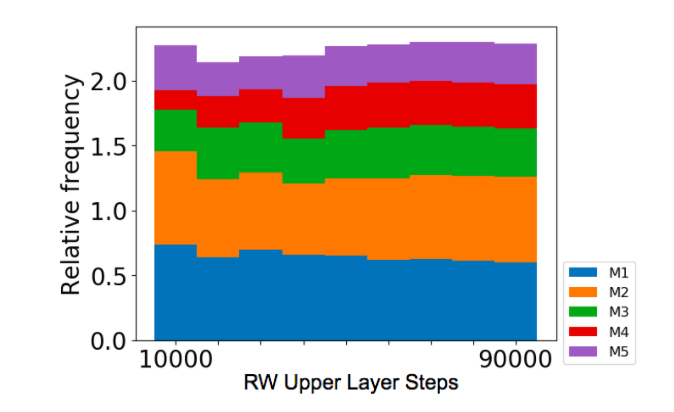}}
    \caption{Exp 1., $\hat{F}$ analysis on Yeast.}
\end{figure}
\begin{figure}[h!]
    \centering
    \subfloat[SSE analysis]{\label{fig:ssemico}
    \includegraphics[width=1.5in,height=1.22in]{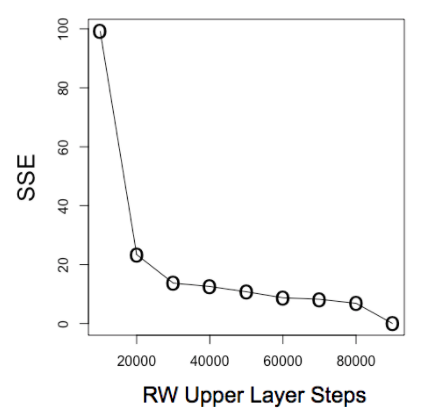} 
    }
    \subfloat[Top 5 convergence]{\label{fig:barmotifs}
    \includegraphics[width=1.8in]{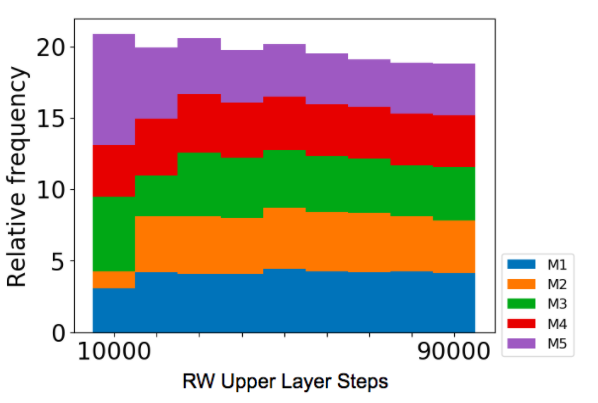}}
    \caption{Exp 2., $\hat{F}$ analysis on MSR.}
\end{figure}
\begin{figure}[h!]
    \subfloat[SSE analysis]{\label{fig:ssedblp1}
    \includegraphics[width=1.5in,height=1.22in]{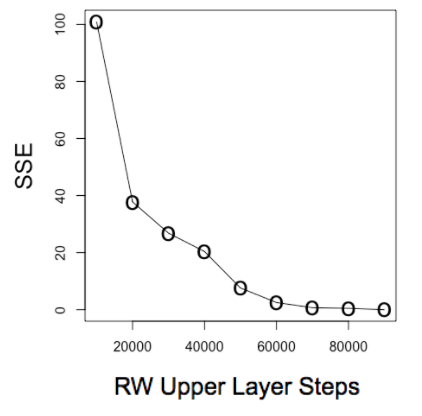}
    }
    \subfloat[Top 5 convergence]{\label{fig:bardblp1}
    \includegraphics[width=1.8in, height=1.22in]{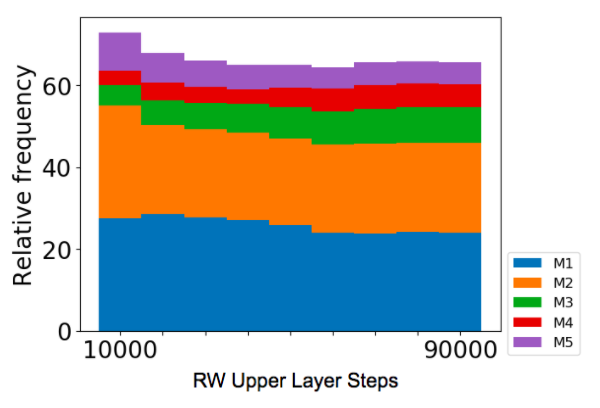}}
    \caption{Exp 3., $\hat{F}$ analysis on DBLP1.}
\end{figure}
\subsubsection{SSE analysis}
First, we will show that our estimator $\hat{F}$ converges to the true value by measuring 
the sum of squared errors (SSE) of the difference between the estimator and the exact value $F$ (eq. \eqref{eq:Fgeneral}) for all the patterns.
The SSE accuracy of {\bf Exp.\ 1} is shown in Figure~\ref{fig:sseyeast}).
The exact computation of eq.~\eqref{eq:Fgeneral} is only possible since the Yeast graph is small. 
The SSE values 
go from 4 to 0.5 after 100k sampled subgraphs (Upper Layer steps), showing that 
our estimator monotonically (and quickly) reduces the estimation error as we sample more subgraphs.
Approximate SSE convergences of {{\bf Exp.~2} and {\bf Exp.~3}
are shown in Figures \ref{fig:ssemico} and \ref{fig:ssedblp1}, respectively. 
In these scenarios, we cannot run the exact algorithm due to
the prohibitive computational cost.
Thus, we calculate $\hat{F}$ (eq.~\eqref{eq:Fapp2}) after
 100k sampled subgraphs and consider it as ground truth.  
The goal is to measure convergence. Again, 
 the SSE values reduces in both test scenarios,
 decreasing from 100 to $0.01$ in both, experiments 2 and 3 (Figs. \ref{fig:ssemico} and \ref{fig:ssedblp1}). 
 This shows that our approach can quickly provide accurate estimates with a relatively few number of subgraph queries.

\subsubsection{Top 5 motifs convergence}
Here, we study
the fluctuation of estimator $\hat{F}$
for the top 5 motifs of MCC, as we collect more samples (Figs. ~\ref{fig:baryeast}, \ref{fig:barmotifs} and \ref{fig:bardblp1}). 
 As we may see,  $\hat{F}$
 converges as we collect more samples, for all the experiments listed in Table \ref{tab:exp-conv}. 
 This result corroborates with our previous
SSE analysis, indicating that the finite-sample bias of our estimator is small.

\subsection{Evaluating the $\hat{\alpha}$ estimator}

We now turn our attention to the accuracy of our estimates of $\hat{\alpha}$ from eq.~\eqref{eq:alphahat} w.r.t.\ the ground-truth value of $\alpha$ in eq.~\eqref{eq:alpha}.
First, we select three representative \CISes $S \in \sS^{(k)}$ 
shown in Figure \ref{fig:subs}, one from
each experiment in table \ref{tab:exp-conv}. 
Then, we perform a detailed evaluation of $\hat{\alpha}$ estimator in each of these experimental scenarios.
Our main goal is to understand the real-world impact of \emph{supernode} size (|$\mathcal{I}_S$|) and \emph{number of tours} $q$ in the accuracy of $\hat{\alpha}$.

\begin{figure}
\subfloat[Yeast's subgraph: this subgraph is composed by proteins with a same biological function]{
\begin{minipage}{.12\textwidth}
\includegraphics[scale=0.13]{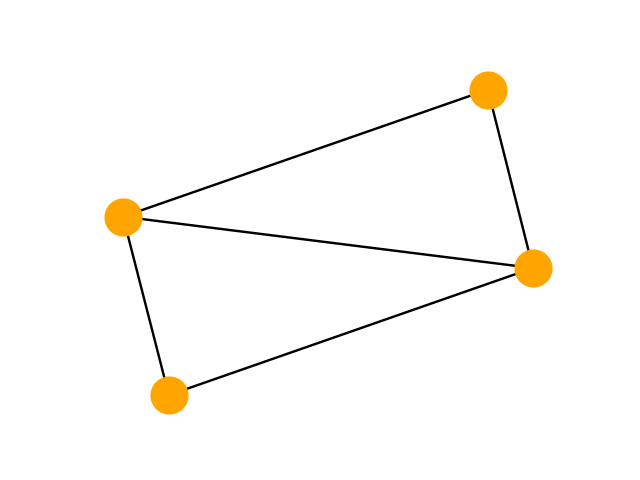}%
\hfill 
\resizebox{0.13\textwidth}{!}{
\begin{minipage}{.13\textwidth}
\begin{tabular}{|l|l|l|}
\hline 
  Node & Degree & Label\\
\hline
898 & 10 & 5 \\ 
968 & 34 & 5 \\
550 & 23 & 5 \\
1046 & 13 & 5 \\
  \hline  
\end{tabular}
\end{minipage}
}
\end{minipage}
\label{fig:sub1}
}\qquad%
\subfloat[MSR's subgraph: the blue label represents authors with senior status]{
\begin{minipage}{.12\textwidth}
\includegraphics[scale=0.13]{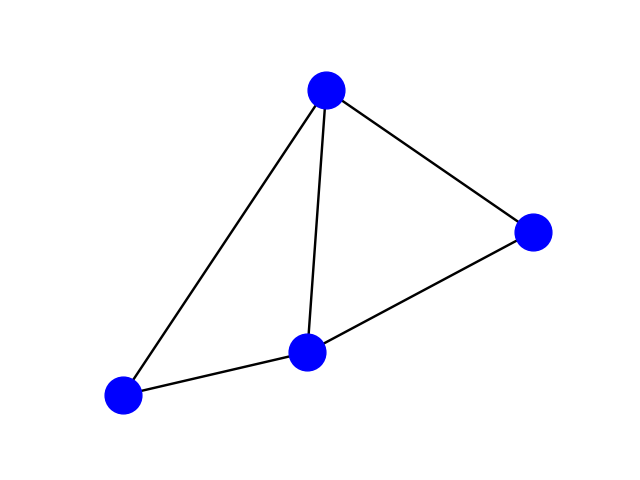}%
\hfill 
\resizebox{0.13\textwidth}{!}{
\begin{minipage}{.13\textwidth}
\begin{tabular}{|l|l|l|}
\hline 
Node & Degree & Label\\
\hline
23819 & 49 & 2 \\
18929 & 134 & 2 \\
12668 & 170 & 2 \\
15181 & 38 & 2 \\
\hline  
\end{tabular}
\end{minipage}
}
\end{minipage}
\label{fig:sub2}
}\qquad
\subfloat[DBLP1's subgraph: the red and blue labels represents author and research topic, respectively]{
\begin{minipage}{.12\textwidth}
\includegraphics[scale=0.13]{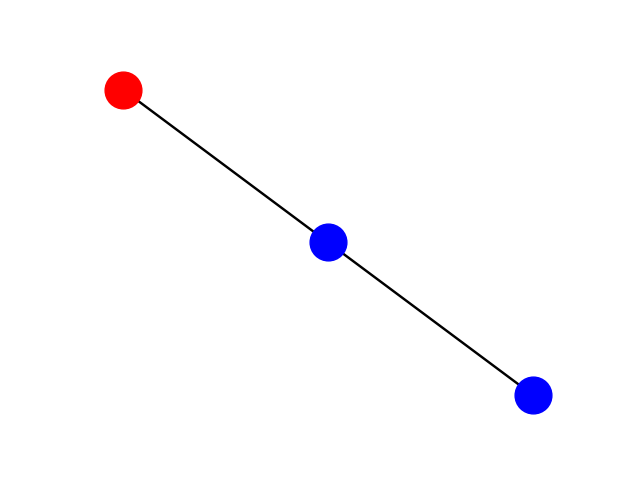}%
\hfill 
\resizebox{0.13\textwidth}{!}{
\begin{minipage}{.13\textwidth}
\begin{tabular}{|l|l|l|}
\hline 
Node & Degree & Label\\
\hline
5018 & 683 & 2\\
4974 & 48 & 2 \\
170  & 24 & 1 \\
  \hline  
\end{tabular}
\end{minipage}
}
\end{minipage}
\label{fig:sub3}
}
\caption{The patterns and other features of the subgraphs considered in our experiments for $\alpha(R,S)$ estimate analysis.} 
\label{fig:subs}
\end{figure}

\begin{figure*}[ht] 
   \centering
   \subfloat[Yeast, Percolation, k=4]{
	\includegraphics[width=2in,height=1.5in]{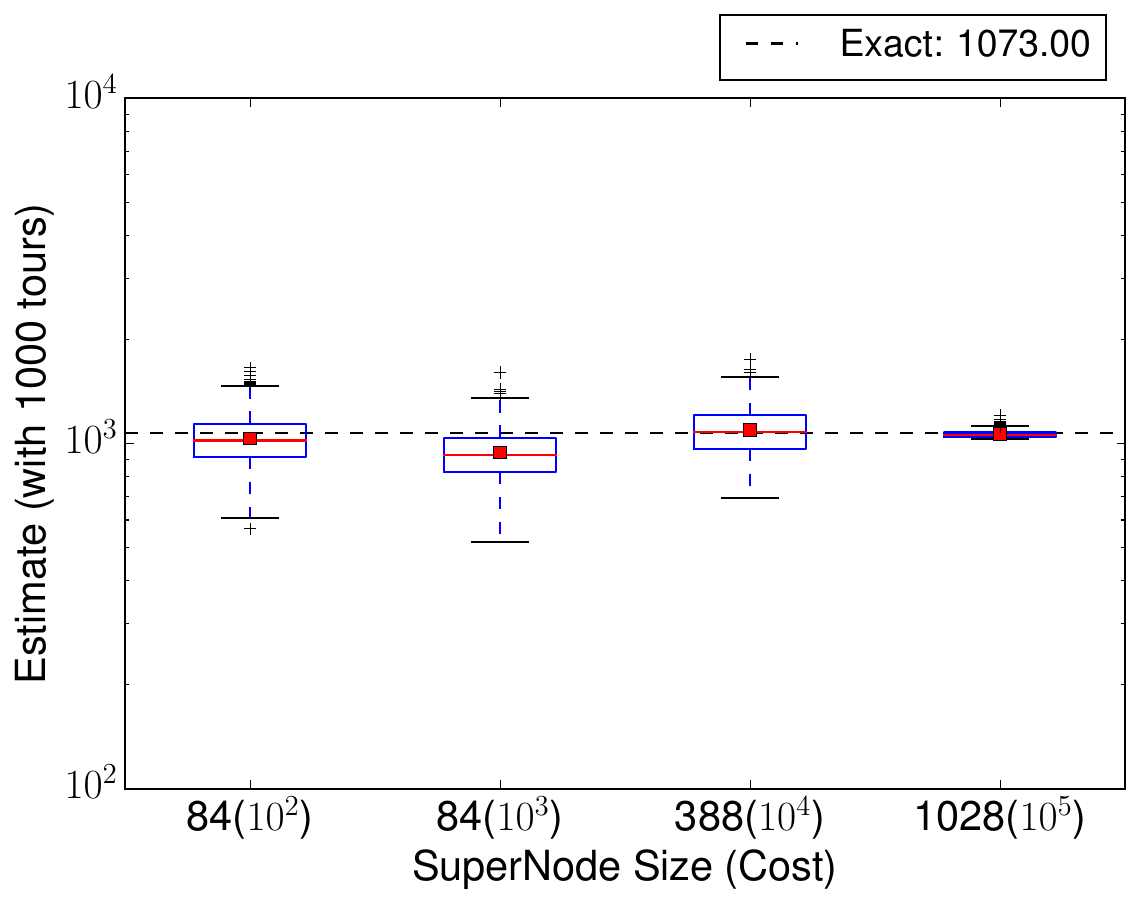}%
   	\label{fig:alphasupernode1}
   }
   ~
   \subfloat[MSR, Percolation, k=4]{
	   \includegraphics[width=2in,height=1.5in]{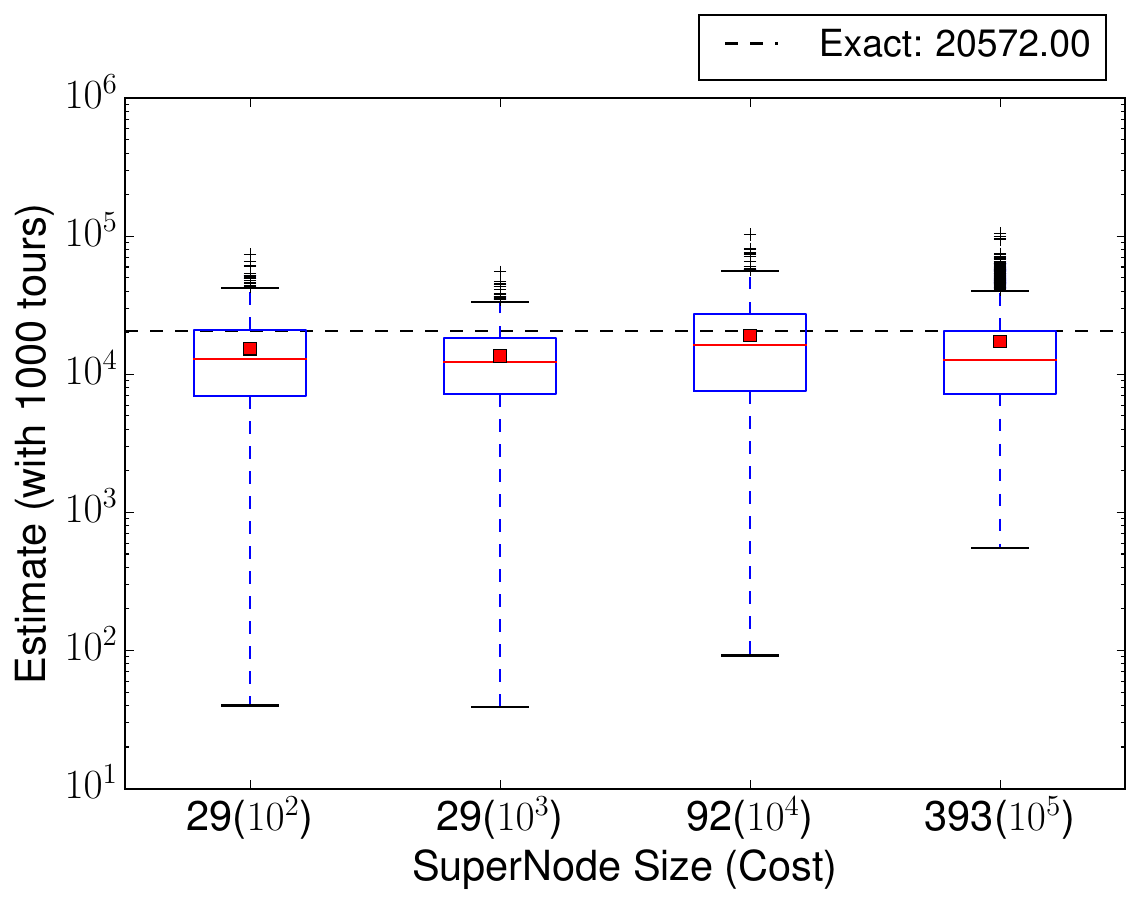}%
   		\label{fig:alphasupernode2}
   }
   ~
   \subfloat[DBLP1, SHDN, k=3]{
	  		   \includegraphics[width=2in,height=1.5in]{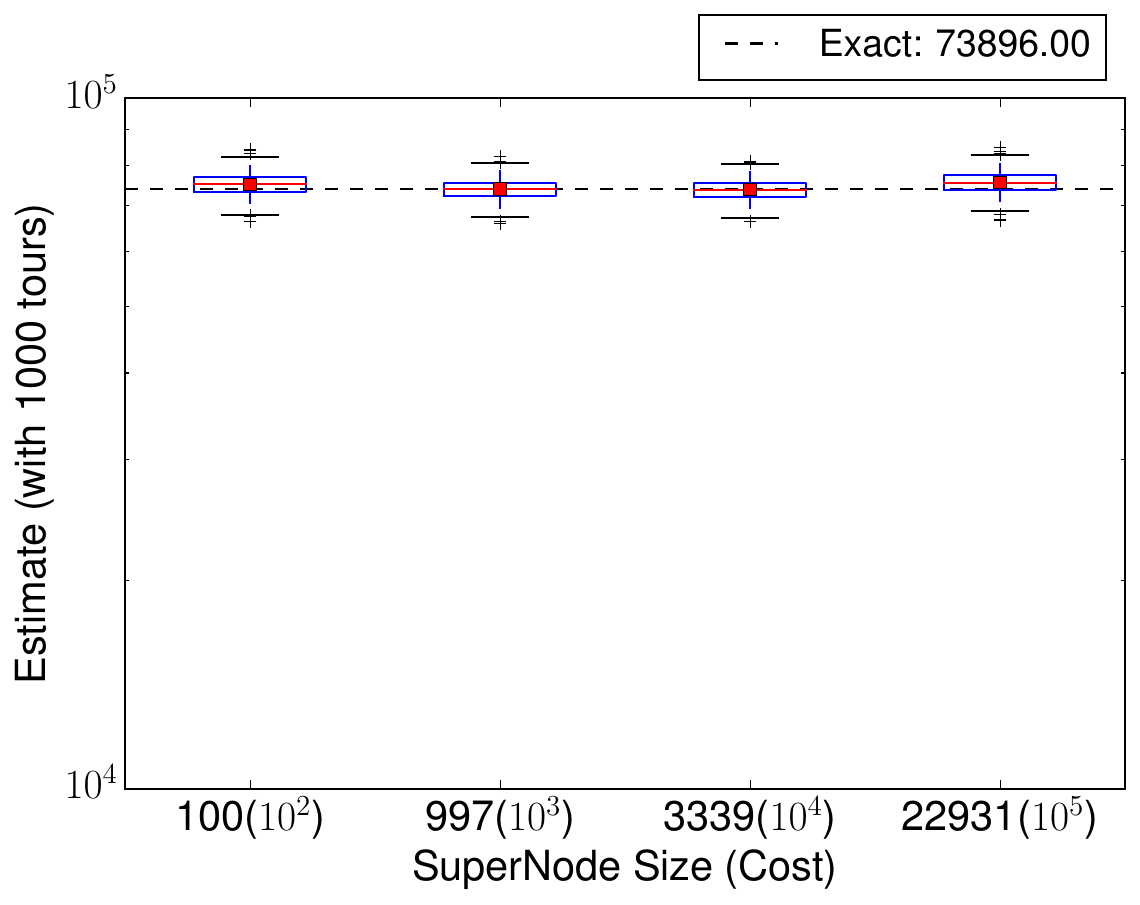}%
   		\label{fig:alphasupernode3}
   }
   \caption{This figure shows the $\alpha_C$ values returned by our algorithm, varying the supernode size.
   Increased supernode size reduces variance in Exp. 1 and 2 (Figs. 5a and 5b), but has almost no effect on Exp. 3 (Fig. 5c).
} 
    	\label{fig:alpha-supernode-size}
\end{figure*}

Figure \ref{fig:alpha-supernode-size} shows
the impact of the supernode size
in the accuracy of estimating $\alpha(R,S)$ for $g(\cdot)=1$ (i.e., we are estimating the class's size, $|C_{\mathrm{R},S}|$). 
The figure shows a boxplot for different supernode sizes and
the vertical dotted lines delimit the exact $\alpha$ (eq. \eqref{eq:alpha}) .
The supernodes were generated using the subgraphs retrieved by 
the iteration-bounded exact algorithm (Alg. \ref{alg:exactbfs}) 
with a budget (described in parentheses). 
Note that, the budget is
 generally larger than $|\mathcal{I}_S|$ 
since the exact algorithm sometimes retrieves 
subgraphs out of the class of interest.
We observe that
larger supernode sizes leads 
to less variance in {\bf Exp.~1} and {\bf Exp.~2}, Figs.\ \ref{fig:alphasupernode1} and \ref{fig:alphasupernode2}, respectively, specially when the budget to 
create them was set to $10^5$, where
we know more elements of the subgraph classes.
On the results for {\bf Exp.\ 3} (Fig. \ref{fig:alphasupernode3}), the error of estimating $\alpha$ is small because the random walks mix fast, 
making the tours shorter for all tested supernode sizes.

\begin{comment}
\begin{figure*}[ht] 
   \centering
   \subfloat[Yeast, Percolation, k=4]{
	\includegraphics[width=2in,height=1.5in]{figs/sigmetrics/outputibmx_EPerco_P4_S1000_T44yeastbigcomp_lg_tgz_gid392840082663508141_steps_box.pdf}%
   	\label{fig:tlengthsupernode1}
   }
   ~
   \subfloat[MSR, Percolation, k=4]{
	\includegraphics[width=2in,height=1.5in]{figs/sigmetrics/outputTAcn_EPerco_P4_S100000_T44authors_lg_tgz.gid9005313744655051248_steps_box.pdf}%
   		\label{fig:tlengthsupernode2}
   }
   ~
   \subfloat[DBLP1, SHDN, k=3]{
	  \includegraphics[width=2in,height=1.5in]{figs/sigmetrics/outputQIz4_EHighDegree_P3_S1000_T44dblp1-noduplicate_txt_tgz_gid2789225271252684_steps_box.pdf}%
   		\label{fig:tlengthsupernode3}
   }
	\label{fig:tlenght-supernode}
   \caption{This figure reports the number of steps (layer 2) necessary to finish a tour (tour length) for different supernode sizes.} 
\end{figure*}
\end{comment}

\begin{figure*}[ht] 
   \centering
   \subfloat[Yeast, Percolation, k=4]{
	\includegraphics[width=2in,height=1.5in]{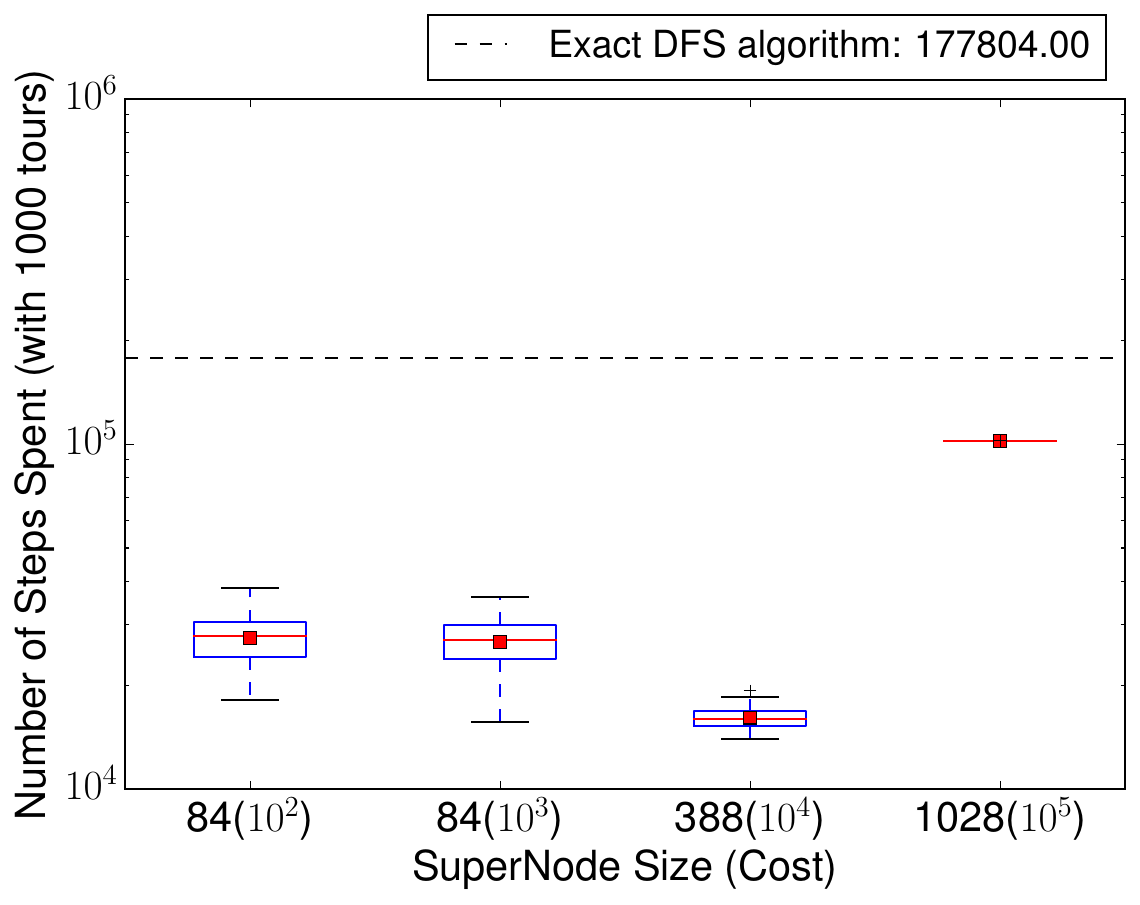}%
   	\label{fig:tlengthsupernode1}
   }
   ~
   \subfloat[MSR, Percolation, k=4]{
	\includegraphics[width=2in,height=1.5in]{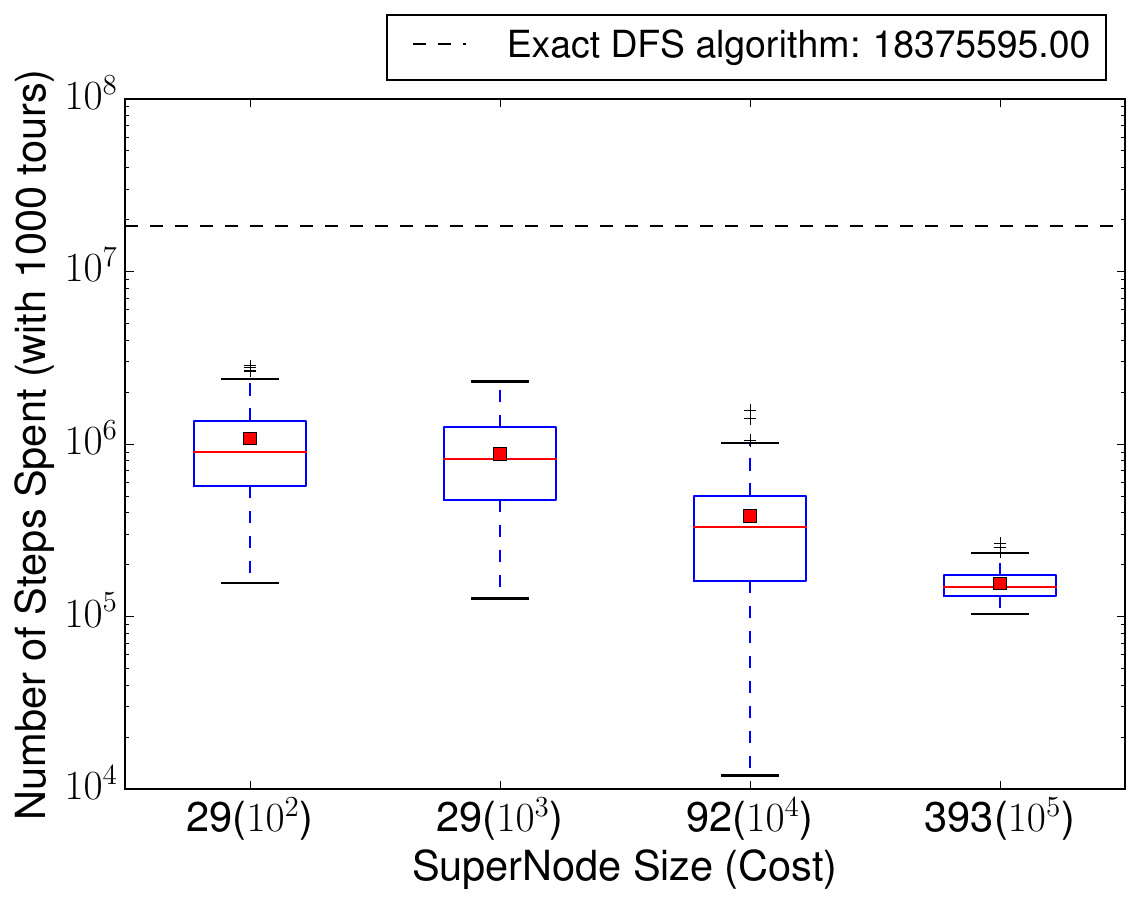}%
   		\label{fig:tlengthsupernode2}
   }
   ~
   \subfloat[DBLP1, SHDN, k=3]{
	  \includegraphics[width=2in,height=1.5in]{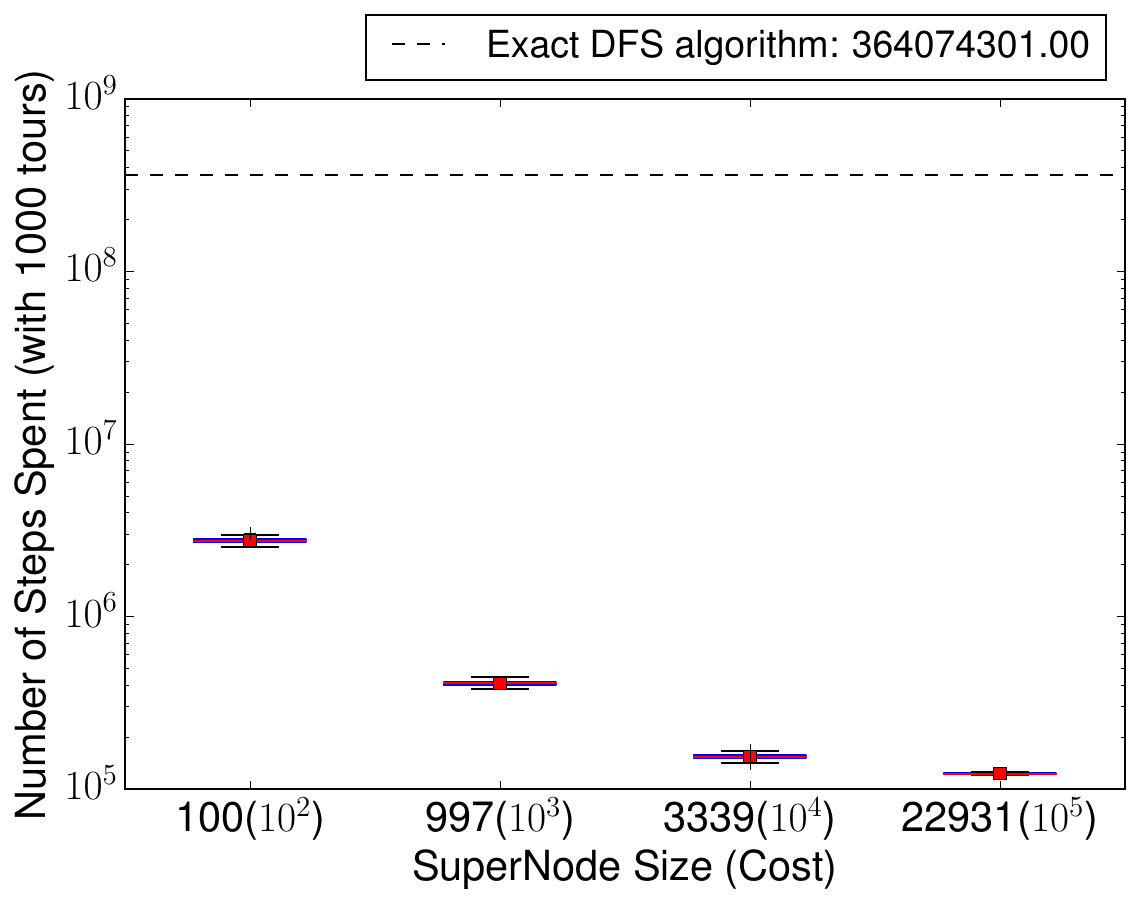}%
   		\label{fig:tlengthsupernode3}
   }
   \caption{
Total cost (sum of RW steps for 1000 tours + iteration-bounded exact computation) to estimate $\alpha(\mathrm{R}, S)$  
for different supernode sizes. 
Larger supernodes boost the efficiency of our method since much less steps are performed in the RW tour. 
The dashed line shows the effort needed to compute the exact BFS-based algorithm.
    } 
   	\label{fig:tlength-supernode}
\end{figure*}

Theoretically, we know that a larger supernode implies shorter tours in average, as the average tour length is inversely proportional to the total number of edges in the supernode. Thus, the initial cost of having spent more computation in the iteration-bounded exact algorithm leads to a larger supernode, which may pay-off if we can get shorter tours. This behavior is seen in Figure~\ref{fig:tlength-supernode}, which measures the total number of subgraph queries needed to complete 1000 tours in our estimator $\hat{\alpha}$ accounting for the exact algorithm retrieves budget $B \in \{10^2,10^3,10^4,10^5\}$. 
The vertical dotted line shows the budget that the exact algorithm required to compute $\alpha(R,S)$ exactly.
Note that use a larger budget in the exact algorithm  
often pays-off in reducing the total number, but of course there is a limit where it start being counter-productive.
Also note that our estimators is generally between one and six order of magnitude faster than the exact computation.

We now study the trade-off between $\hat{\alpha}$'s accuracy and the number of tours. Figure~\ref{fig:alphatour2} shows the estimate of $\alpha(R,S)$ when the supernode is fixed 
1k iterations (steps)
and the number of tours varies between 10 and 1k. In fact,  
we may see that the estimate of $\hat{\alpha}$ not only improves when we use more tours,
but also it converges to true value (as shown by the boxplots).
For practitioners, we advise to test the estimator variance for different number of tours in their 
specific applications: a diminishing return in variance reduction indicates a good number of tours.

\begin{figure*}[ht] 
   \centering
   \subfloat[Yeast, Percolation, k=4]{
	\includegraphics[width=2in,height=1.5in]{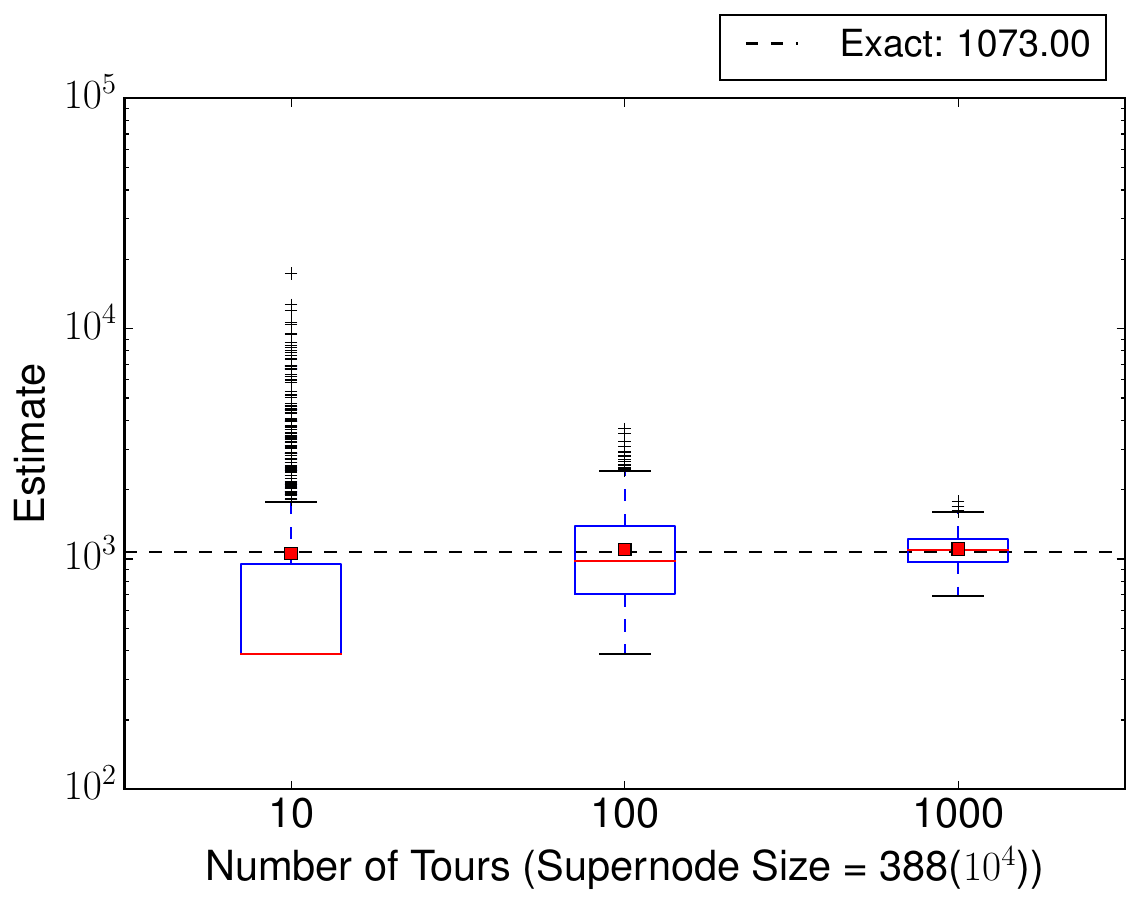}%
   	\label{fig:alphtour1}
   }
   \subfloat[MSR, Percolation, k=4]{
	   \includegraphics[width=2in,height=1.5in]{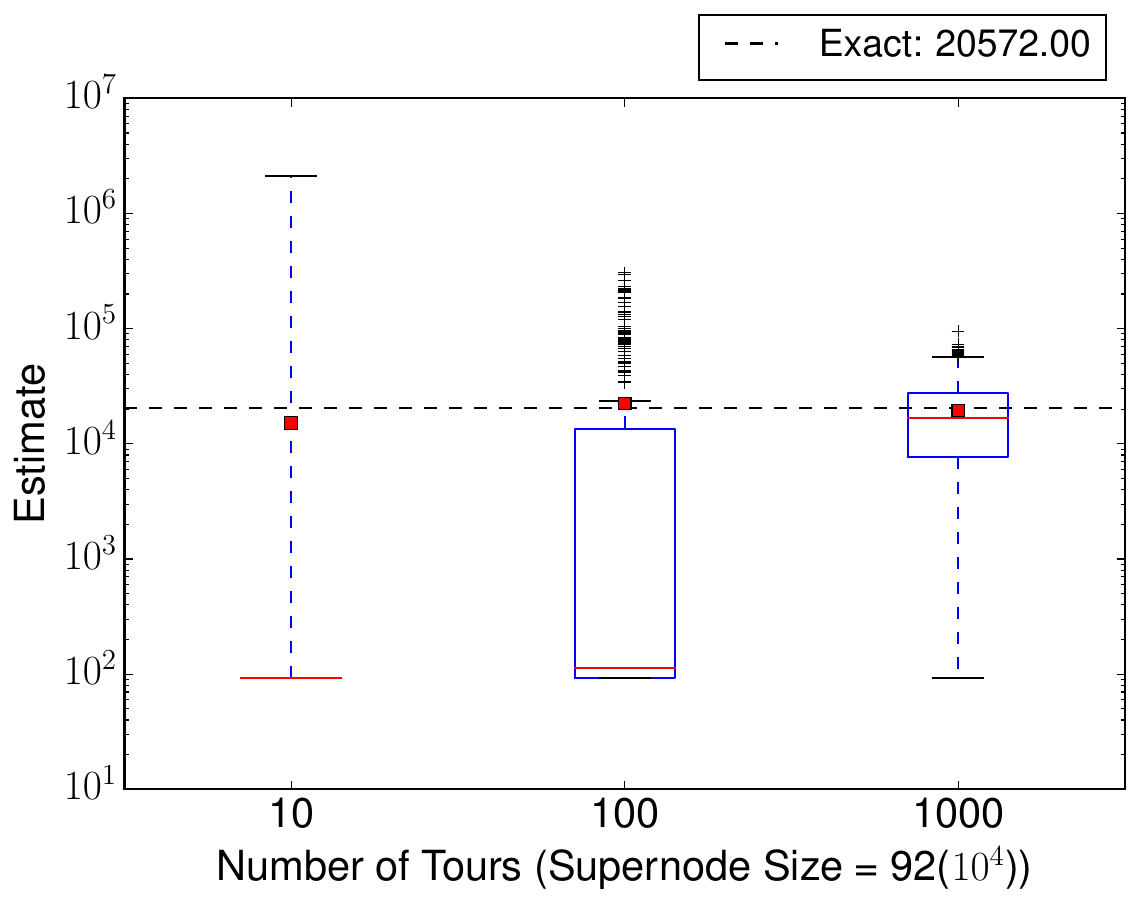}%
   		\label{fig:alphatour2}
   }
   \subfloat[DBLP1, SHDN, k=3]{
	   \includegraphics[width=2in,height=1.5in]{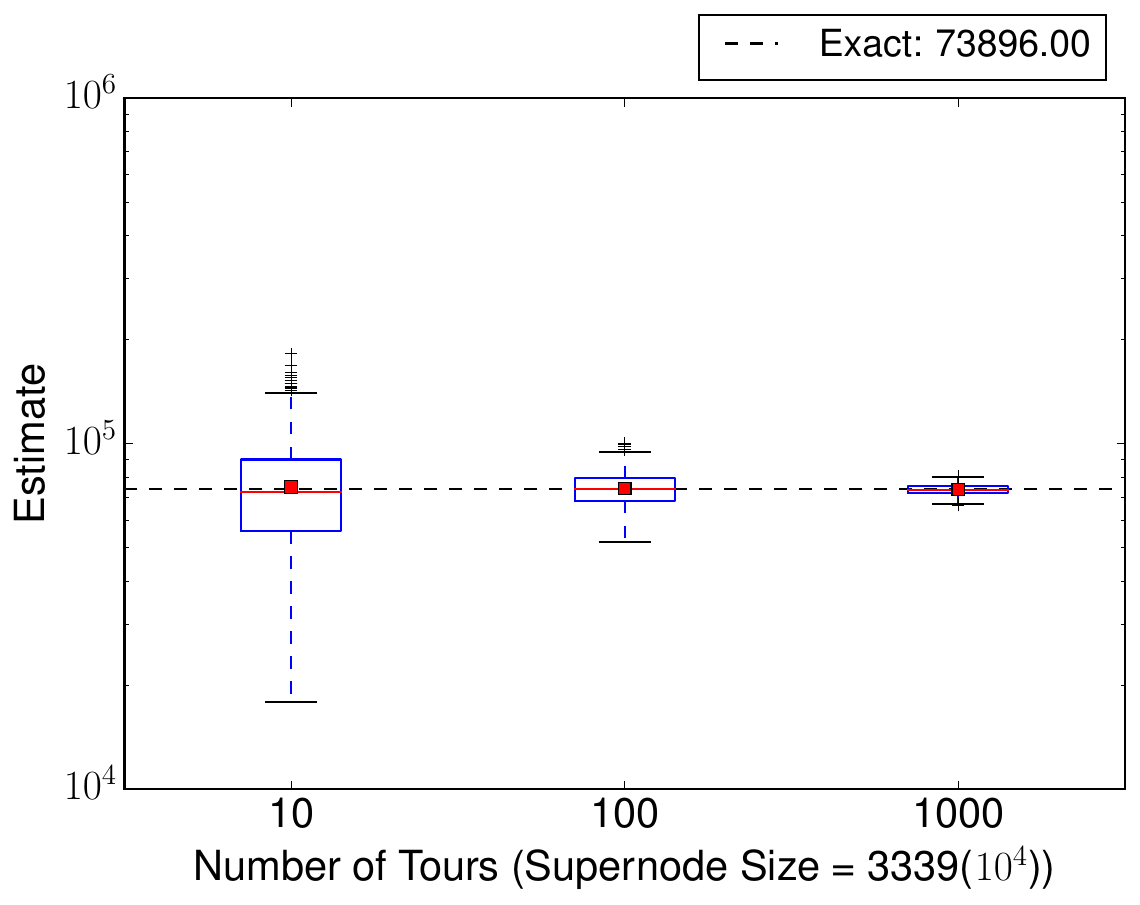}%
   		\label{fig:alphatour3}
   }
   \caption{Accuracy evaluation of $\alpha_C$, varying the number of tours performed. The budget to build the supernode set was fixed to 1k iterations (steps).
   As expected, the variance reduces as the supernode size increases, improving the quality of the estimator.} 
	\label{fig:alpha-tour-size}
\end{figure*}

\subsection{Scalability and efficiency}
We assess the
efficiency and 
scalability of \name from two perspectives.
First, we calculate the cost 
to compute $\alpha$ for a given class $C_{\mathrm{R},S}$
-- in terms of the
number of steps -- for both the exact DFS method, \name and the best case scenario. 
The later may be represented by an oracle algorithm which
 enumerates only the subgraphs that are really necessary ($C_{\mathrm{R},S}$). 

\begin{table}[!htb]
\centering
\begin{tabular}{r|rrrr} \toprule
Exp. & \multicolumn{1}{c}{\#Tours}  & \#Steps/Tour  &    \#Steps (Total)   & Speedup  \\ \midrule
\parbox[t]{2mm}{\multirow{4}{*}{\rotatebox[origin=c]{90}{Exp. 1}}}
&10.00&26.97$\pm$0.93&369.65&481.00x\\
&100.00&27.30$\pm$0.30&2830.45&62.82x\\
&1000.00&27.44$\pm$0.09&27536.12&6.46x\\
\cline{2-5}
& \textit{Best \#Steps} & 1073
& \textit{DFS \#Steps} & 177804\\
\midrule \midrule
\parbox[t]{2mm}{\multirow{4}{*}{\rotatebox[origin=c]{90}{Exp. 2}}} 
&10.00&832.0$\pm$94.16&9320.08&1971.61x\\
&100.00&846.92$\pm$30.69&85691.88&214.44x\\
&1000.00&860.88$\pm$9.61&861880.03&21.32x\\
\cline{2-5}
& \textit{Best \#Steps} & 20572
& \textit{DFS \#Steps} & 18375595\\
\midrule \midrule
\parbox[t]{2mm}{\multirow{4}{*}{\rotatebox[origin=c]{90}{Exp. 3}}} 
&10.00&411.46$\pm$2.85&5114.57&71183.73x\\
&100.00&409.43$\pm$0.89&41942.52&8680.32x\\
&1000.00&409.48$\pm$0.28&410475.19&886.96x\\
\cline{2-5}
& \textit{Best \#Steps} & 73896
& \textit{DFS \#Steps} & 364074301 \\ 
\bottomrule
\end{tabular}
\caption{
Comparison between our approach, DFS solution and the best case scenario in terms of the number of steps.
}
\label{tab:performance}
\end{table}

Table \ref{tab:performance} shows a
 comparison between the algorithms in {\bf Exp.~2} (with similar results for Exp.\ 1 and 3).
Our approximate solution provides 
significant computational savings (up to 3 orders-of-magnitude), outperforming the exact method in running time. 
 As expected, the performance gap between these algorithms decreases
as we increase the number of tours used to estimate $\alpha(R,S)$. There is an ``accuracy vs. efficiency'' trade-off, since the estimate gets more accurate as
we increase the number of tours or supernode size.
%

\begin{comment}
\begin{figure}
    \centering
    \includegraphics[width=2in,height=1.5in]{plot-speed}\qquad
    \caption{Scalability of \name w.r.t. number of cores}
    \label{fig:speedup}
\end{figure}
\end{comment}

Table \ref{tab:run} shows the execution time 
and respective speedups of \name on 
the three datasets. In all of them, we compute in parallel the statistics for 
10k subgraphs (i.e., the number of subgraphs sampled in the upper lower RW).
 The remaining configurations are the same used in the previous experiments (Table \ref{tab:exp-conv}).
As we may realize, 
\name provides a near-linear speedup in all of our tests, showing the benefits 
of our parallelization strategy. The super-linear behaviour of 
\name in the Yeast graph is explained by the graph being small enough to fit in the memory cache.

\begin{comment}
\begin{table}[hbtp]
\centering
\begin{tabular}{|c|r|r|r|}
\hline %
  11
& 5m6s
&338m49s
& 1002m16s
\\ \midrule
  22    
& 2m36s
& 176m32s
& 522m26s
\\ \midrule
44
& 1m27s
& 125m54s
& 322m19s 
\\
\midrule \midrule
Speedup 
&  
x3.5
&
x2.7
&
x3.1
\\ \bottomrule
\end{tabular}
\caption{\name's running time analysis}
\label{tab:run}
\end{table}
\end{comment}

\begin{table}[hbtp]
\centering
\begin{tabular}{ l|r|r|r|c}
\toprule
\multirow{2}{*}{\textbf{Dataset}} & \multicolumn{3}{|c|}{\textbf{\#cores}} & \multirow{2}{*}{\textbf{Speedup}} \\ \cline{2-4} 
   & \multicolumn{1}{|c|}{11} & \multicolumn{1}{|c|}{22} & \multicolumn{1}{|c|}{44} &
\\ \midrule 
  Yeast 
& 5m6s
& 2m36s
& 1m27s
& x3.5
\\ \midrule
  MSR  
&338m49s
&176m32s
&125m54s
& x2.7
\\ \midrule
DBLP1 
& 1002m16s
& 522m26s
& 322m19s
& x3.1
\\ \bottomrule
\end{tabular}
\caption{\name's running time analysis}
\label{tab:run}
\end{table}
\section{Related Work}\label{s:related}

Existing methods for estimating induced subgraph statistics of graphs are focused on simple subgraph counting tasks. 
For instance,
 edge-based algorithms sample edges and combine 
 them to estimate subgraphs counts~\cite{ahmed2014graph} and
triangles~\cite{jha2013space}.
There are drawbacks to these methods: (a) the inability to consider larger subgraphs, 
(b) the cost to remove the bias of sampling a particular edge and, (c) the need of having access to the entire network.
Methods based on random walks were
 proposed to address these challenges.

There are RW solutions focused on estimating simple graph statistics~\cite{Ribeiro2010, hardiman2013estimating}.
In a complementary direction, Hasan and Zaki~\cite{AlHasan:2009} proposed a MCMC 
algorithm that works on chain of subgraph patterns 
from a input graph, where their goal was 
sample a subset of interesting patterns uniformly.
Wang et al.~\cite{Wang:2014} 
and Bhuiyan et al.~\cite{Guise:2014} develop two sampling methods 
to deal with motifs and graphlet counting, 
 using RWs and high-order networks. Such solutions 
 differ in the way they remove the bias from sampling:
 one uses Horvitz-Thompson estimator and the second employs a Metropolis-Hastings algorithm.

Recently,
MCMC methods sample subgraphs by walking 
in the actual graph structure
rather than high-order networks~\cite{han2016waddling}.
For that, they keep not only the the most recent visit node, 
 but the last $k$ nodes visited by the RW. 
Subgraphs are 
 derived from these nodes and 
  the bias of sampling them are computed according  
 pre-established equations.
Although walking 
in the original graph may be faster than 
 to perform MCMC in a high-order network, 
these methods are not amenable to the subgraph relationships and classes introduced in our work.

\section{Conclusions}

This paper introduces the concept of subgraph relations and 
its applications to graph pattern mining and learning problems.
Relations generalize traditional GPM problems and they can aid the analysis of subgraph patterns.
In particular, we saw that (1) subgraph relations can help reduce learning biases associated with locally subgraph-dense regions (e.g., high-degree nodes) and (2) subgraph classes may provide 
 novel and interesting analysis on motifs in graphs.
Finally, we show that our proposed hybrid exact-sampling estimator 
is consistent, accurate and significantly faster than the exact approach alone.

\noindent \newline
{\bf Acknowledgements.} 
This work was partially supported by 
CNPq, CAPES, Fapemig,  INCT-Cyber, MasWeb, EuBra-BigSea, Atmosphere
and by the ARO, under the U.S. Army Research 
Laboratory award W911NF-09-2-0053.
Any opinions, findings and 
conclusions or recommendations expressed in this material are those of 
the authors and do not necessarily reflect the views of ARO.

\balance
\bibliographystyle{IEEEtran}
\bibliography{references}

\end{document}